%% file: author.tex
\documentclass[oribibl]{svproc}

\usepackage[sort&compress, numbers]{natbib}
\usepackage[utf8]{inputenc}
\usepackage[T1]{fontenc}    
\usepackage[pdftex]{graphicx}
\usepackage{amssymb,amsmath}
\usepackage{algorithm}
\usepackage{algpseudocode}
\usepackage[dvipsnames]{xcolor}
\usepackage{url}
\usepackage{booktabs}
\usepackage{tabularx}
\usepackage{multirow}
\usepackage[font=footnotesize]{caption}
\usepackage{wrapfig}
\usepackage{enumitem} 
\usepackage{xspace}
\usepackage{bm}
\usepackage{bbm} 

\usepackage{amsthm} 
\usepackage{thmtools}
\usepackage{mathtools}
\usepackage{nicefrac}
\usepackage{tikz}
\usetikzlibrary{shapes,positioning,intersections,quotes}
\usepackage{caption,subcaption}
\DeclareGraphicsExtensions{.jpg,.pdf,.mps,.png,.eps} 
\graphicspath{{./figs/}}
\input{macros/math_operators}

\input{macros/math_defns}

\input{macros/text_shortcuts}
\input{macros/tikz_macros.tex}
\usepackage[pdftex,bookmarks=true]{hyperref} 
\hypersetup{letterpaper,bookmarksopen,bookmarksnumbered,
pdfpagemode=UseOutlines,
colorlinks=true,
linkcolor=blue,
anchorcolor=blue,
citecolor=blue,
filecolor=blue,
menucolor=blue,
urlcolor=blue
}
\usepackage[capitalise]{cleveref} 
\newcommand{\eref}[1]{(\ref{#1})} 
\newcommand{\sref}[1]{Section~\ref{#1}} 
\newcommand{\figref}[1]{Fig.~\ref{#1}} 
\newcommand{\tabref}[1]{Table~\ref{#1}} 
\newcommand{\algoref}[1]{Algorithm~\ref{#1}} 
\newcommand{\thmref}[1]{Theorem~\ref{#1}}
\newcommand{\lemref}[1]{Lemma~\ref{#1}}

\newcommand{\appref}[1]{Appendix~\ref{#1}}

\newtheorem{assum}{Assumption}[section]

\newcommand{\kaynote}[1]{{\xxnote{KK}{BurntOrange}{#1}}}

\newcommand{\xxnote}[3]{}
\ifx\hidenotes\undefined
  \renewcommand{\xxnote}[3]{\color{#2}{#1: #3}}
\fi
\newcommand{\fullFigGap}[0]{\vspace{-1.5\baselineskip}} 

\begin{document}
\mainmatter              
%
\title{Imitation Learning as f-Divergence Minimization}
\titlerunning{f-Imitation}  
\author{Liyiming Ke\inst{1} \and Sanjiban Choudhury\inst{1} \and Matt Barnes \inst{1} \and
Wen Sun \inst{2} \and Gilwoo Lee \inst{1}  \and Siddhartha Srinivasa \inst{1}}
\authorrunning{Ke et al.} 
\tocauthor{} 
\institute{Paul G. Allen School of Computer Science \& Engineering, University of Washington. Seattle WA 98105, USA,\\
\email{\{kayke,sanjibac,mbarnes,gilwoo,siddh\}@cs.washington.edu},\\
\and
The Robotics Institute, Carnegie Mellon University, Pittsburgh PA 15213, USA, \\
\email{wensun@andrew.cmu.edu}}

\maketitle              

\input{inputs/abstract}

\input{inputs/introduction}

\input{inputs/related_work}

\input{inputs/problem_formulation}

\input{inputs/general_approach}

\input{inputs/divergence}

\input{inputs/experiments}

\input{inputs/discussion}

\input{inputs/acknowledge}


\renewcommand\bibpreamble{\vspace{-3\baselineskip}}

\bibliographystyle{unsrtnat}
{\small
\bibliography{reference}
}
\clearpage
\appendix
\onecolumn
\begin{center}
    \LARGE{Appendix for ``Imitation Learning \\ as f-Divergence Minimization''}
\end{center}

\input{appendix/state_action_distribution}
\newpage
\input{appendix/upper_action_bound}
\newpage
\input{appendix/existing_algorithms}
\newpage
\input{appendix/rkl_action_distribution}
\newpage
\input{appendix/dre}
\newpage
\input{appendix/divide_by_zero.tex}

\end{document}

%% file: macros/math_operators.tex
\DeclareMathOperator*{\argmin}{arg\,min}

\newcommand{\argminprob}[1]{\underset{#1}{\argmin}}

\newcommand{\abs}[1]{\left|#1 \right|}

\newcommand{\expect}[2]{\mathbb{E}_{#1}\left[#2\right]}

\newcommand{\real}[0]{\mathbb{R}}
\newcommand{\realPositive}[0]{\real^+}
\newcommand{\Natural}[0]{\mathbb{N}}
\newcommand{\NaturalPositive}[0]{\Natural^+}

\newcommand{\bbm}{\begin{bmatrix}}
\newcommand{\ebm}{\end{bmatrix}}

\newcommand{\pair}[2]{\left( #1, #2\right)}

\newcommand{\tuple}[1]{\left\langle #1\right\rangle}

\newcommand{\Ind}[0]{\mathbb{I}}
\newcommand{\overbar}[1]{\mkern 1.5mu\overline{\mkern-1.5mu#1\mkern-1.5mu}\mkern 1.5mu}

\newcommand{\bigo}[1]{\mathcal{O}\left(#1\right)}
\newcommand{\dom}[1]{\mathrm{dom}_{#1}}

%% file: macros/math_defns.tex
\newcommand{\state}[0]{s}
\newcommand{\stateSpace}[0]{\mathcal{S}}
\newcommand{\action}[0]{a}
\newcommand{\actionSpace}[0]{\mathcal{A}}
\newcommand{\transFnDef}[0]{P}
\newcommand{\transFn}[3]{\transFnDef \left(#3 | #1, #2 \right)}

\newcommand{\hor}[0]{T}

\newcommand{\policy}[0]{\pi}
\newcommand{\policyClass}[0]{\Pi}

\newcommand{\traj}[0]{\tau}
\newcommand{\trajSeg}[2]{\tau_{#1,#2}}

\newcommand{\dist}[0]{\rho}
\newcommand{\initDist}[0]{\dist_0}
\newcommand{\distTraj}[1]{\dist_{#1}}
\newcommand{\distState}[1]{\dist_{#1}}
\newcommand{\distStateTime}[2]{\dist_{#1}^{#2}}

\newcommand{\lossFn}[0]{\ell}
\newcommand{\expert}[0]{\policy^*}

\newcommand{\policyLearn}[0]{\hat{\policy}}

\newcommand{\badStateSpace}[0]{\stateSpace_\mathrm{bad}}

\newcommand{\fDiv}[2]{D_f\left(#1,#2\right)}
\newcommand{\klDiv}[2]{D_{\mathrm{KL}}\left(#1,#2\right)}
\newcommand{\rklDiv}[2]{D_{\mathrm{RKL}}\left(#1,#2\right)}

\newcommand{\tvDiv}[2]{D_{\mathrm{TV}}\left(#1,#2\right)}

\newcommand{\policyBar}[0]{\overbar{\policy}}

%% file: macros/text_shortcuts.tex
\newcommand{\Dagger}[0]{\textsc{DAgger}\xspace}

\newcommand{\fIL}[1]{\emph{#1}--\textsc{VIM}\xspace}

\newcommand{\ifIL}[1]{\emph{#1}--\emph{i}\textsc{VIM}\xspace}

%% file: macros/tikz_macros.tex
\newcommand{\DrawExpert}{ 
    \draw [CornflowerBlue, line width=0.7mm] plot [smooth] coordinates { (0,-1) (2,1) (4,0)};
    \draw [CornflowerBlue, line width=0.7mm] plot [smooth] coordinates { (0,-1) (2,1.1) (4,-0.1)};
    \draw [CornflowerBlue, line width=0.7mm] plot [smooth] coordinates {(0,-1) (2,1.2) (0,4)};
    \draw [CornflowerBlue, line width=0.7mm] plot [smooth] coordinates {(0,-1) (1.8,1.2) (0.2,4)};
    \draw [CornflowerBlue, line width=0.7mm] plot [smooth] coordinates {(0,-1) (1.9,1.2) (0,3.9)};
}

\newcommand{\DrawPolicyLeft}{
    \draw [WildStrawberry, line width=0.3mm] plot [smooth] coordinates {(0,-1) (2,1.4) (0.1,4.1)};
    \draw [WildStrawberry, line width=0.3mm] plot [smooth] coordinates {(0,-1) (1.8,1.2) (0.1,3.8)};
    \draw [WildStrawberry, line width=0.3mm] plot [smooth] coordinates {(0,-1) (1.9,1.2) (-0.1,3.9)};
}

\newcommand{\DrawPolicyRight}[1]{
    \draw [WildStrawberry, line width=0.3mm] plot [smooth] coordinates { (0,-1) (2,1.2) (4.1,0.2-#1)};
    \draw [WildStrawberry, line width=0.3mm] plot [smooth] coordinates { (0,-1) (2,1.2) (3.9,-0.2-#1)};
}

\newcommand{\DrawPolicyCenter}{
    \draw [WildStrawberry, line width=0.3mm] plot [smooth] coordinates { (0,-1) (2,1.2) (4,1)};
    \draw [WildStrawberry, line width=0.3mm] plot [smooth] coordinates { (0,-1) (2,1.2) (4,1.8)};
    \draw [WildStrawberry, line width=0.3mm] plot [smooth] coordinates { (0,-1) (2,1.2) (3.5,3.5)};
    \draw [WildStrawberry, line width=0.3mm] plot [smooth] coordinates { (0,-1) (2,1.2) (2.5,4)};
    \draw [WildStrawberry, line width=0.3mm] plot [smooth] coordinates { (0,-1) (2,1.2) (1.8,4.3)};
}

\newcommand{\DrawEnv}{
    \node at (0, -1)[circle,fill,inner sep=1.5pt]{};
    \node[below left=0pt of {(0,-1)}] {\tiny $s_0$};
    \node [cloud, fill=ForestGreen, draw,cloud puffs=10,cloud puff arc=120, aspect=1, inner ysep=0em] at (3,2.3) {\tiny Tree};
}

%% file: inputs/abstract.tex

\begin{abstract}
We address the problem of imitation learning with multi-modal demonstrations. 
Instead of attempting to learn all modes, we argue that in many tasks it is sufficient to imitate any one of them. We show that the state-of-the-art methods such as GAIL and behavior cloning, due to their choice of loss function, often incorrectly interpolate between such modes. Our key insight is to minimize the right divergence between the learner and the expert state-action distributions, namely the reverse KL divergence or I-projection. We propose a general imitation learning framework for estimating and minimizing any $f$-Divergence. 
By plugging in different divergences, we are able to recover existing algorithms such as Behavior Cloning (Kullback-Leibler), GAIL (Jensen Shannon) and 
\Dagger (Total Variation).
 Empirical results show that our approximate I-projection technique is able to imitate multi-modal behaviors more reliably than GAIL and behavior cloning.
\keywords{machine learning, imitation learning, probabilistic reasoning}
\end{abstract}


%

%% file: inputs/introduction.tex

\section{Introduction}
\label{sec:introduction}

We study the problem of imitation learning from demonstrations that have \emph{multiple modes}. This is often the case for tasks with multiple, diverse near-optimal solutions. Here the expert has no clear preference between different choices (e.g. navigating left or right around obstacles~\cite{ross2013learning}). Imperfect human-robot interface also lead to variability in inputs (e.g. kinesthetic demonstrations with robot arms~\cite{finn2016guided}). Experts may also vary in skill, preferences and other latent factors. We argue that in many such settings, it suffices to learn a single mode of the expert demonstrations to solve the task. How do state-of-the-art imitation learning approaches fare when presented with multi-modal inputs?

Consider the example of imitating a racecar driver navigating around an obstacle. 
The expert sometimes steers left, other times steers right. 
What happens if we apply behavior cloning~\cite{Pomerleau88} on this data?
The learner policy (a Gaussian with fixed variance) interpolates between the modes and drives into the obstacle.


Interestingly, this oddity is not restricted to behavior cloning. \cite{li2017infogail} show that a more sophisticated approach, GAIL~\cite{ho2016generative}, also exhibits a similar trend. Their proposed solution, InfoGAIL~\cite{li2017infogail}, tries to recover all the latent modes and learn a policy for each one. For demonstrations with several modes, recovering all such policies will be prohibitively slow to converge. 

\begin{figure}[!tb]
\centering
\includegraphics[width=0.9\columnwidth]{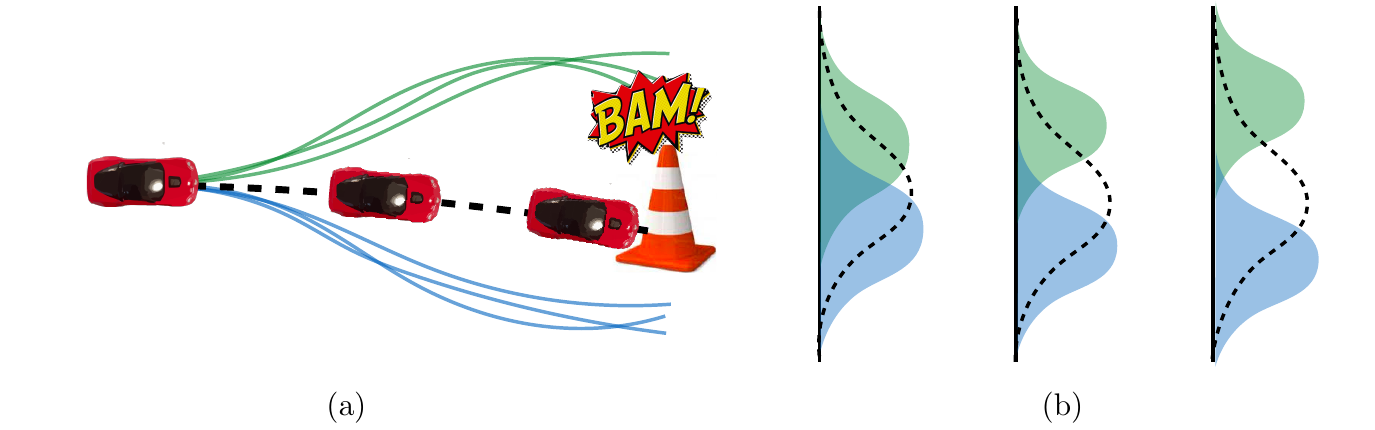}
\caption{Behavior cloning fails with multi-modal demonstrations. Experts go left or right around obstacle. Learner interpolates between modes and crashes into obstacle.\fullFigGap}
\label{fig:intro}
\end{figure}

Our key insight is to view imitation learning algorithms as minimizing divergence between the expert and the learner trajectory distributions. Specifically, we examine the family of $f$-divergences. Since they cannot be minimized exactly, we adopt estimators from \cite{nowozin2016f}. We show that behavior cloning minimizes the Kullback-Leibler (KL) divergence (M-projection), GAIL minimizes the Jensen-Shannon (JS) divergence and \Dagger minimizes the Total Variation (TV). Since both JS and KL divergence exhibit a \emph{mode-covering} behavior, they end up interpolating across modes. On the other hand, the reverse-KL divergence (I-projection) has a \emph{mode-seeking} behavior and elegantly collapses on a subset of modes fairly quickly.

The contributions and organization of the remainder of the paper 
are:
\begin{enumerate}
\item We introduce a unifying framework for imitation learning as minimization of $f$-divergence between learner and trajectory distributions (\sref{sec:formulation}).
\item We propose algorithms for minimizing estimates of any $f$-divergence. Our framework is able to recover several existing imitation learning algorithms for different divergences. We closely examine reverse KL divergence and propose efficient algorithms for it (\sref{sec:approach}).
\item We argue for using reverse KL to deal with multi-modal inputs (\sref{sec:divergence}). We empirically demonstrate that reverse KL collapses to one of the demonstrator modes on both bandit and RL environments, whereas KL and JS unsafely interpolate between the modes (\sref{sec:experiments}).
\end{enumerate}

%% file: inputs/related_work.tex
\section{Related Work}
\label{sec:related_work}

Imitation learning (IL) has a long-standing history in robotics as a tool to program desired skills and behavior in autonomous machines~\cite{ROB-053,argall2009survey,billard2016learning,Bagnell-2015-5921}. Even though IL has of late been used to bootstrap reinforcement learning (RL)~\cite{ross2014reinforcement,sun2017deeply,sun2018truncated,cheng2018fast,rajeswaran2017learning}, we focus on the original problem where an extrinsic reward is not defined. We ask the question -- ``what objective captures the notion of similarity to expert demonstrations?''. Note that this question is orthogonal to other factors such as whether we are model-based / model-free or whether we use a policy / trajectory representation. 

IL can be viewed as supervised learning where the learner selects the same action as the expert (referred to as behavior cloning~\cite{pomerleau1989alvinn}). However small errors lead to large distribution mismatch. This can be somewhat alleviated by interactive learning, such as \Dagger~\cite{ross2011reduction}. Although shown to be successful in various applications~\cite{ross2013learning,kim2013learning,gupta2017cognitive}, there are domains where it's impractical to have on-policy expert labels~\cite{laskey2017iterative,laskey2016shiv}. More alarmingly, there are counter-examples where the \Dagger objective results in undesirable behaviors~\cite{laskey2017comparing}. We discuss this further in \appref{sec:existing_algorithms}.

Another way is to view IL as recovering a reward (IRL)~\cite{ratliff2009learning,ratliff2006maximum} or Q-value~\cite{piot2017bridging} that makes the expert seem optimal. Since this is overly strict, it can be relaxed to value matching which, for linear rewards, further reduces to matching feature expectations~\cite{abbeel2004apprenticeship}. Moment matching naturally leads to maximum entropy formulations~\cite{ziebart2008maximum} which has been used successfully in various applications~\cite{finn2016guided,wulfmeier2015maximum}. Interestingly, our divergence estimators also match moments suggesting a deeper connection.

The degeneracy issues of IRL can be alleviated by a game theoretic framework where an adversary selects a reward function and the learner must compete to do as well as the expert~\cite{syed2008game,ho2016model}. Hence IRL can be connected to min-max formulations~\cite{finn2016connection} like GANs~\cite{goodfellow2014generative}. GAIL~\cite{ho2016generative}, SAM~\cite{blonde2018sample} uses this to directly recover policies. AIRL~\cite{fu2017learning}, EAIRL~\cite{qureshi2018adversarial} uses this to recover rewards. This connection to GANs leads to interesting avenues such as stabilizing min-max games~\cite{peng2018variational}, learning from pure observations~\cite{torabi2018generative,torabi2018behavioral,peng2018sfv} and links to f-divergence minimization~\cite{nowozin2016f,nguyen2010estimating}. 

In this paper, we view IL as $f$-divergence minimization between learner and expert. Our framework encompasses methods that look at specific measures of divergence such as minimizing relative entropy~\cite{boularias2011relative} or symmetric cross-entropy~\cite{Rhinehart_2018_ECCV}. Note that \cite{Ghasemipour2019} also independently arrives at such connections between f-divergence and IL.\footnote{Different from \cite{Ghasemipour2019}, our framework optimizes \emph{trajectory} divergence.} We particularly focus on multi-modal expert demonstrations which has generally been treated by clustering data and learning on each cluster~\cite{babes2011apprenticeship,dimitrakakis2011bayesian}. InfoGAN~\cite{chen2016infogan} formalizes the GAN framework to recover latent clusters which is then extended to IL~\cite{hausman2017multi,li2017infogail}. MCTE~\cite{lee2018maximum} extended maximum entropy formulations with casual Tsallis entropy to learn sparse multi-model policy using sparse mixture density net~\cite{lee2018sparse}. ~\cite{belousov2017f} studied how choice of divergence affected policy improvement for reinforcement learning. Here, we look at the role of divergence with multi-model expert demonstrations.


%% file: inputs/problem_formulation.tex

\section{Problem Formulation}
\label{sec:formulation}

\paragraph{Preliminaries}
We work with a finite horizon Markov Decision Process (MDP) $\tuple{\stateSpace, \actionSpace, \transFnDef, \initDist, \hor}$ where $\stateSpace$ is a set of states, $\actionSpace$ is a set of actions, and $\transFnDef$ is the transition dynamics. $\initDist(\state)$ is the initial distribution over states and $\hor \in \NaturalPositive$ is the time horizon. In IL paradigm, the MDP does not include a reward function.

We examine stochastic policies $\policy(\action | \state) \in [0,1]$.
Let a trajectory be a sequence of state-action pairs $\traj = \{ \state_0, \action_1, \state_1, \dots, \action_\hor, \state_\hor\}$. It induces a distribution of trajectories $\distTraj{\policy}(\traj)$ and state $\distStateTime{\policy}{t}(\state)$ as:
\begin{equation}
\begin{split}
	\distTraj{\policy}(\traj) = \initDist(\state_0) \prod \limits_{t=1}^{\hor} \policy(\action_t | \state_{t-1}) \transFn{\state_{t-1}}{\action_t}{\state_{t}} \\
	\distStateTime{\policy}{t}(\state) = \sum\limits_{\state',\action} \distStateTime{\policy}{t-1}(\state')
\policy(\action | \state') \transFn{\state}{\action}{\state'} 
\end{split}
\end{equation}

The average state distribution across time $\distState{\policy}(\state) = \frac{1}{\hor} \sum_{t=1}^{\hor} \distStateTime{\policy}{t-1}(\state)$\footnote{Alternatively $\distState{\policy}(\state) = \sum_{\traj} \distTraj{\policy}(\traj) \left( \frac{1}{T}  \sum_{t=1}^T \Ind(\state_{t-1} = \state) \right)$. Refer to \thmref{thm:average_state_dist} in Appendix~\ref{sec:rkl_action_distribution}}.

\paragraph{The f-divergence family}
Divergences, such as the well known Kullback-Leibler (KL) divergence, measure differences between probability distributions. We consider a broad class of such divergences called \emph{f-divergences}~\cite{csiszar2004information,liese2006divergences}. Given probability distributions $p(x)$ and $q(x)$ over a finite set of random variables $X$, such that $p(x)$ is absolutely continuous w.r.t $q(x)$, we define the f-divergence:
\begin{equation}
\label{eq:prob:fdiv_def}
 \fDiv{p}{q} = \sum_x q(x) f \left( \frac{p(x)}{q(x)} \right)
\end{equation}
where $f: \realPositive \rightarrow \real$ is a convex, lower semi-continuous function. Different choices of $f$ recover different different divergences, e.g.\ KL, Jensen Shannon or Total Variation (see~\cite{nowozin2016f} for a full list).

\paragraph{Imitation learning as f-divergence minimization}

Imitation learning is the process by which a learner tries to behave similarly to an expert based on inference from demonstrations or interactions. There are a number of ways to formalize ``similarity'' (\sref{sec:related_work}) -- either  as a classification problem where learner must select the same action as the expert~\cite{ross2011reduction} or as an inverse RL problem where learner recovers a reward to explain expert behavior~\cite{ratliff2009learning}. Neither of the formulations is error free.

We argue that the metric we actually care about is matching the distribution of trajectories $\distTraj{\expert}(\traj) \approx \distTraj{\policy}(\traj)$. 
One such reasonable objective is to minimize the $f$-divergence between these distributions
\begin{equation}
\label{eq:prob:fdiv_il_min}
 \policyLearn = \argminprob{\policy \in \policyClass} \; \fDiv{\distTraj{\expert}(\traj)}{\distTraj{\policy}(\traj)} = \argminprob{\policy \in \policyClass} \sum_\traj \distTraj{\policy}(\traj) f \left( \frac{\distTraj{\expert}(\traj)}{\distTraj{\policy}(\traj)} \right)
\end{equation}
Interestingly, different choice of $f$-divergence leads to different learned policies (more in \sref{sec:divergence}).

Since we have only sample access to the expert state-action distribution, the divergence between the expert and the learner has to be estimated.
However, we need many samples to accurately estimate the trajectory distribution as the size of the trajectory space grows exponentially with time, i.e. $\bigo{\abs{\stateSpace}^\hor}$.
Instead, we can choose to minimize the divergence between the \emph{average state-action distribution} as the following:
\begin{equation}
\begin{split}
\label{eq:prob:sa_fdiv_il_min}
 \policyLearn &= \argminprob{\policy \in \policyClass} \; \fDiv{\distState{\expert}(\state)\expert(\action | \state)}{\distState{\policy}(\state)\policy(\action | \state)}  \\
 &= \argminprob{\policy \in \policyClass}  \; \sum_{\state, \action}  \distState{\policy}(\state) \policy(\action | \state) f \left( \frac{\distState{\expert}(\state) \expert(\action | \state)}{\distState{\policy}(\state) \policy(\action | \state)} \right)
 \end{split}
\end{equation}

We show that this lower bounds the original objective, i.e. trajectory distribution divergence.
\vspace{-.25em}
\begin{theorem}[Proof in Appendix~\ref{sec:state_action_distribution}]
\label{thm:state_action_distribution}
Given two policies $\pi$ and $\expert$, the f-divergence between trajectory distribution is lower bounded by f-divergence between average state-action distribution.
\begin{equation*}
\fDiv{\distTraj{\expert}(\traj)}{\distTraj{\policy}(\traj)} \geq \fDiv{\distTraj{\expert}(\state)\expert(\action | \state)}{\distTraj{\policy}(\state)\policy(\action | \state)}
\end{equation*}
\end{theorem}

%% file: inputs/general_approach.tex

\section{Framework for Divergence Minimization}
\label{sec:approach}

The key problem is that we don't know the expert policy $\expert$ and only get to observe it. Hence we are unable to compute the divergence exactly and must instead \emph{estimate} it based on \emph{sample} demonstrations. We build an estimator which lower bounds the state-action, and thus, trajectory divergence. The learner then minimizes the estimate.

\subsection{Variational approximation of divergence}
Let's say we want to measure the $f$-divergence between two distributions $p(x)$ and $q(x)$. Assume they are unknown but we have i.i.d samples, i.e., $x\sim p(x)$  and $x\sim q(x)$. Can we use these to estimate the divergence? \cite{nguyen2010estimating} show that we can indeed estimate it by expressing $f(\cdot)$ in it's \emph{variational form}, i.e. $f(u) = \sup_{t \in \mathrm{dom}_{f^*}} \left( tu - f^*(t) \right)$, where $f^*(\cdot)$ is the convex conjugate~\footnote{For a convex function $f(\cdot)$, the convex conjugate is $f^*(v) = \sup_{u \in \mathrm{dom}_f} \left( uv - f(u) \right)$. Also $(f^*)^* = f$.} Plugging this in the expression for $f$-divergence~\eref{eq:prob:fdiv_def} we have
\begin{equation}
\begin{aligned}
\label{eq:prob:fdiv_lb}
  \fDiv{p}{q} &= \sum_x q(x) f \left( \frac{p(x)}{q(x)} \right) = \sum_x q(x) \sup_{t \in \mathrm{dom}_{f^*}} \left( t\frac{p(x)}{q(x)} - f^*(t) \right) \\
  & \geq \sup_{\phi \in \Phi} \sum_x q(x) \left( \phi(x) \frac{p(x)}{q(x)} - f^*(\phi(x)) \right) \\
  & \geq \sup_{\phi \in \Phi} \left( \underbrace{\expect{x \sim p(x)}{\phi(x)}}_{\text{\tiny sample estimate}} - \underbrace{\expect{x \sim q(x)}{ f^* ( \phi(x) )}}_{\text{\tiny sample estimate}} \right) \\
\end{aligned}
\end{equation}
Here $\phi: X \rightarrow \dom{f^*}$ is a function approximator which we refer to as an \emph{estimator}. The lower bound is both due to Jensen's inequality  and the restriction to an estimator class $\Phi$. Intuitively, we convert divergence estimation to a discriminative classification problem between two sample sets. 

How should we choose estimator class $\Phi$? We can find the optimal estimator $\phi^*$ by taking the variation of the lower bound \eref{eq:prob:fdiv_lb} to get $\phi^*(x) = f'\left( \frac{p(x)}{q(x)} \right)$. Hence $\Phi$ should be flexible enough to approximate the subdifferential $f'(.)$ \emph{everywhere}. Can we use neural networks discriminators~\cite{goodfellow2014generative} as our class $\Phi$? \cite{nowozin2016f} show that to satisfy the range constraints, we can parameterize $\phi(x) = g_f(V_w(x))$ where $V_w : X \rightarrow \real$ is an unconstrained discriminator and $g_f: \real \rightarrow \dom{f^*}$ is an \emph{activation function}. We plug this in \eref{eq:prob:fdiv_lb} and the result in \eref{eq:prob:sa_fdiv_il_min} to arrive at the following problem.

\begin{algorithm}[t]
\caption{\fIL{f} \label{alg:fgail}}
\begin{algorithmic}[1]
\State Sample trajectories from expert $\tau^* \sim \distTraj{\expert}$
\State Initialize learner and estimator parameters $\theta_0$, $w_0$
\For{$i=0$ \textbf{to} $N-1$}
\State Sample trajectories from learner $\tau_i \sim \distTraj{\policy_{\theta_i}}$
\State Update estimator
\Statex \hspace{3ex} $w_{i+1} \leftarrow w_{i} + \eta_w \nabla_w \left( \sum_{ \pair{\state}{\action} \in \tau^* } g_f(V_w(\state, \action)) - \sum_{\pair{\state}{\action} \in \tau_i } f^*(g_f(V_w(\state, \action))) \right)$
\State Apply policy gradient
\Statex \hspace{3ex} $\theta_{i+1} \gets \theta_i - \eta_\theta \sum_{ \pair{\state}{\action} \sim \tau_i } \nabla_\theta \; \log \policy_\theta (\action | \state) Q^{f^*(g_f(V_w))}(\state, \action) $
\Statex \hspace{3ex} where $Q^{f^*(g_f(V_w))}(\state_{t-1}, \action_t) = - \sum\limits_{i=t}^T f^*(g_f(V_w(\state_{i-1}, \action_i)))$
\EndFor
\State \textbf{Return} $\policy_{\theta_N}$ 
\end{algorithmic}
\end{algorithm}

\begin{problem}[Variational Imitation (VIM)]
\label{prob:fil} 
Given a divergence $f(\cdot)$, compute a learner $\policy$ and discriminator $V_w$ as the saddle point of the following optimization
\begin{equation}
\label{eq:fil}
\policyLearn = \argminprob{\policy \in \policyClass} \max_{w} \; \expect{ \pair{\state}{\action} \sim \distTraj{\expert} }{g_f(V_w(\state, \action))} - \expect{ \pair{\state}{\action} \sim \distTraj{\policy} }{ f^*(g_f(V_w(\state, \action)))} 
\end{equation}
where $ \pair{\state}{\action} \sim \distTraj{\expert}$ are sample expert demonstrations, $\pair{\state}{\action} \sim \distTraj{\policy}$ are samples learner rollouts.
\end{problem}
We propose the algorithmic framework \fIL{f} (\algoref{alg:fgail}) which solves \eref{eq:fil} iteratively by updating estimator $V_w$ via supervised learning and learner $\theta_i$ via policy gradients. \algoref{alg:fgail} is a meta-algorithm. Plugging in different $f$-divergences (\tabref{table:fdiv}), we have different algorithms
\begin{enumerate}
	\item \fIL{KL}: Minimizing forward KL divergence
	\begin{equation}
	\label{eq:klvim}
\policyLearn = \argminprob{\policy \in \policyClass} \max_{w} \; \expect{ \pair{\state}{\action} \sim \distTraj{\expert} }{V_w(\state, \action)} - \expect{ \pair{\state}{\action} \sim \distTraj{\policy} }{ \exp(V_w(\state, \action)-1)} 
\end{equation}
	\item \fIL{RKL}: Minimizing reverse KL divergence (removing constant factors)
	\begin{equation}
		\label{eq:rklvim}
\policyLearn = \argminprob{\policy \in \policyClass} \max_{w} \; \expect{ \pair{\state}{\action} \sim \distTraj{\expert} }{-\exp(-V_w(\state, \action))} + \expect{ \pair{\state}{\action} \sim \distTraj{\policy} }{- V_w(\state, \action)} 
\end{equation}
	\item \fIL{JS}: Minimizing Jensen-Shannon divergence
	\begin{equation}
		\label{eq:jsvim}
\policyLearn = \argminprob{\policy \in \policyClass} \max_{w} \; \expect{ \pair{\state}{\action} \sim \distTraj{\expert} }{\log D_w(\state, \action)} - \expect{ \pair{\state}{\action} \sim \distTraj{\policy} }{ \log (1 - D_w(\state, \action))} 
\end{equation}
where $D_w(\state, \action) = (1 + \exp(-V_w(\state,\action))^{-1}$.
\end{enumerate}

\subsection{Recovering existing imitation learning algorithms}
Various existing IL approaches can be recovered under our framework. We defer the readers to \appref{sec:existing_algorithms} for deductions and details.

\textbf{Behavior Cloning~\cite{Pomerleau88} -- Kullback-Leibler (KL) divergence.} We show that the policy minimizing the KL divergence of trajectory distribution can be $\policyLearn = -\mathbb{E}_{\state \sim \distState{\expert}, \action \sim \expert(\cdot |\state)} \log(\policy(\action | \state))$, which is equivalent to behavior cloning with a cross entropy loss for multi-class classification.

\textbf{Generative Adversarial Imitation Learning (GAIL)~\cite{ho2016generative} -- Jensen-Shannon (JS) divergence.}
We see that JS-VIM~\eref{eq:jsvim} is exactly the GAIL optimization (without the entropic regularizer).

\begin{table}[t]
\label{table:fdiv}
\scriptsize
\caption{List of $f$-Divergences used, conjugates, optimal estimators and activation function} \label{table:fdiv}
\begin{tabularx}{\textwidth}{lcccc}
\toprule
{\bf Divergence}
& $f(u)$
& $f^*(t)$
& $\phi^*(x)$
& $g_f(v)$ \\
\midrule
Kullback-Leibler 
& $u\log u$ 
& $\exp (t-1)$ 
& $1 + \log \frac{p(x)}{q(x)}$ 
& $v$ \\
Reverse KL 
& $-\log u$ 
& $-1 - \log(-t)$ 
& $-\frac{q(x)}{p(x)}$ 
& $-\exp(v)$ \\
Jensen-Shannon
& \begin{tabular}{@{}c@{}}$-(u+1)\log \frac{1+u}{2} +$  \\ $ u \log u$\end{tabular} 
& $- \log(2 - \exp (t))$ 
& $\log \frac{2p(x)}{p(x) + q(x)}$ 
&  \begin{tabular}{@{}c@{}}$-\log(1+\exp(-v)) +$ \\ $\log(2)$\end{tabular}\\
Total Variation
& $\frac{1}{2} |u-1| $
& $t$
& $\frac{1}{2} \text{sign}(\frac{p(x)}{q(x)} - 1)$ 
& $\frac{1}{2} \text{tanh}(v)$ \\
\bottomrule
\end{tabularx}
\end{table}

\textbf{Dataset Aggregation (\Dagger)~\cite{ross2011reduction} -- Total Variation (TV) distance.} Using Pinsker's inequality and the fact that TV is a \emph{distance metric}, we have the following upper bound on TV
\begin{equation*}
\begin{aligned}
	\tvDiv{ \distTraj{\expert}(\traj) }{ \distTraj{\policy}(\traj) } & \leq \hor \expect{\state \sim \distState{\policy}(\state)}{ \tvDiv{\expert(\action|\state)}{\policy(\action|\state)} } \\
	& \leq \hor \sqrt{ \expect{\state \sim \distState{\policy}(\state)}{ \klDiv{\expert(\action|\state)}{\policy(\action|\state)} } }
\end{aligned}
\end{equation*}
\Dagger solves this non i.i.d problem in an iterative supervised learning manner with an interactive expert. Counter-examples to \Dagger~\cite{laskey2017comparing} can now be explained as an artifact of this divergence.

\subsection{Alternate techniques for Reverse KL minimization via interactive learning}
We highlight the Reverse KL divergence which has received relatively less attention in IL literature. \fIL{RKL}~\eref{eq:rklvim} has some shortcomings. First, it's a double lower bound approximation due to \thmref{thm:state_action_distribution} and Equation \eref{eq:prob:fdiv_lb}. Secondly, the optimal estimator is a state-action density ratio which maybe quite complex (\tabref{table:fdiv}). Finally, the optimization \eref{eq:fil} may be slow to converge.

However, assuming access to an \emph{interactive expert}, i.e.we can query an interactive expert for any $\expert(\action | \state)$, we can exploit Reverse KL divergence:

\begin{equation}
\begin{aligned}
\rklDiv{\distTraj{\expert}(\traj)}{\distTraj{\policy}(\traj)} &= T\mathbb{E}_{s\sim \rho_{\pi}} [D_{RKL}(\pi^*(\cdot|s), \pi(\cdot|s))]\\
&= \hor \; \expect{\state \sim \distState{\policy}}{ \sum_\action \policy(\action | \state) \log \frac{\pi(\action | \state)}{\expert(\action | \state)}} \notag
\end{aligned}
\end{equation}

Hence we can directly minimize action distribution divergence. Since this is on states induced by $\policy$, this falls under the regime of \emph{interactive learning}~\cite{ross2011reduction} where we query the expert on \emph{states visited by the learner}. We explore two different interactive learning techinques for I-projection, deferring to \appref{sec:rkl_action_distribution} and \appref{app:i_RKL} for details.

\textbf{Variational action divergence minimization.} 
Apply the \fIL{RKL} but on \emph{action divergence}:
\begin{equation}
\begin{aligned}
\policyLearn &= \argminprob{\policy \in \policyClass} \; \expect{\state \sim \distState{\policy}} { \expect{ \action \sim \expert(. | \state) }{- \exp(V_w(\state, \action))} + \expect{  \action \sim \policy(. | \state) }{ V_w(\state, \action)} } \\
\end{aligned}
\end{equation}
Unlike \fIL{RKL}, we collect a fresh batch of data from \emph{both} an interactive expert and learner every iteration. We show that this estimator is far easier to approximate than \fIL{RKL}(\appref{sec:rkl_action_distribution}).

\textbf{Density ratio minimization via no regret online learning.} We first upper bound the action divergence:
\begin{equation}
\begin{aligned}
\label{eq:dre_opt}
  \rklDiv{\distTraj{\expert}(\traj)}{\distTraj{\policy}(\traj)} &= T\expect{\state \sim \distState{\policy}}{ \expect{\action \sim \policy(.|\state)}{ \log \frac{\policy(\action | \state)}{\expert(\action | \state)}}} \\
  & \leq T\expect{\state \sim \distState{\policy}}{ \expect{\action \sim \policy(.|\state)}{  \frac{ \policy(\action | \state) }{ \expert(\action | \state) } - 1 }} \notag\\
\end{aligned}
\end{equation}
Given a batch of data from an interactive expert and the learner, we invoke an off-shelf density ratio estimator (DRE)~\cite{kanamori2012statistical} to get $\hat{r}(s,a) \approx \frac{\rho_{\pi}(s)\pi(a|s)}{\rho_{\pi}(s)\pi^*(a|s)} = \frac{\pi(a|s)}{\pi^*(a|s)}$. Since the optimization is a non i.i.d learning problem, we solve it by dataset aggregation. Note this \emph{does not require invoking policy gradients}. In fact, if we choose an expressive enough policy class, this method gives us a global performance guarantee which neither GAIL or any \fIL{f} provides (\appref{app:i_RKL}).

%% file: inputs/divergence.tex
\section{Multi-modal Trajectory Demonstrations}
\label{sec:divergence}

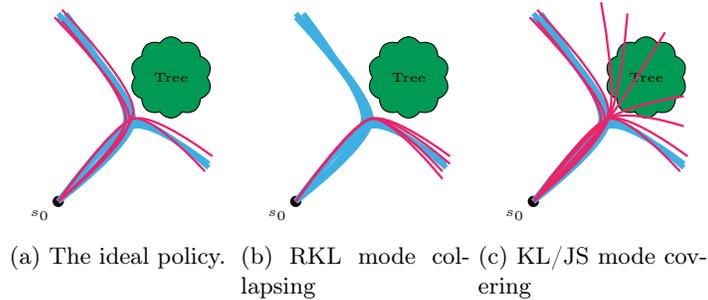
\begin{figure}[t]
\centering
\begin{subfigure}[t]{0.25\textwidth}
\centering
\begin{tikzpicture}[scale=0.5]
    \DrawEnv
    \DrawExpert
    \DrawPolicyLeft
    \DrawPolicyRight{0}
\end{tikzpicture}
\caption{The ideal policy.}
\label{fig:cartoon-ideal}
\end{subfigure}
\begin{subfigure}[t]{0.25\textwidth}
\centering
\begin{tikzpicture}[scale=0.5]
    \DrawEnv
    \DrawExpert
    \DrawPolicyRight{0}
    \DrawPolicyRight{0.2}
\end{tikzpicture}
\caption{RKL mode collapsing} 
\label{fig:cartoon-collapse}
\end{subfigure}
\begin{subfigure}[t]{0.25\textwidth}
\centering
\begin{tikzpicture}[scale=0.5]
    \DrawEnv
    \DrawExpert
    \DrawPolicyRight{0}
    \DrawPolicyLeft
    \DrawPolicyCenter
\end{tikzpicture}
\caption{KL/JS mode covering}
\label{fig:cartoon-cover}
\end{subfigure}
\caption{Illustration of the safety concerns of mode-covering behavior. (a) Expert demonstrations and policy roll-outs are shown in blue and red, respectively. (b) RKL receives only a small penalty for the safe behavior whereas KL receives an infinite penalty. (c) The opposite is true for the unsafe behavior where learner crashes.\fullFigGap}
\label{fig:divergences-cartoon}
\end{figure}


We now examine multi-modal expert demonstrations. Consider the demonstrations in \figref{fig:divergences-cartoon} which avoid colliding with a tree by turning left or right with equal probability.
Depending on the policy class, it may be impossible to achieve zero divergence for \emph{any} choice of $f$-divergence (\figref{fig:cartoon-ideal}), e.g., $\policyClass$ is Gaussian with fixed variance.  Then the question becomes, if the globally optimal policy in our policy class achieves non-zero divergence, how should we design our objective to fail elegantly and safely? In this example, one can imagine two reasonable choices: (1) replicate one of the modes (mode-collapsing) or (2) cover both the modes plus the region between (mode-covering). We argue that in some imitation learning tasks when the dominant mode is desirable, paradigm (1) is preferable.
\vspace{-1em}
\paragraph{Mode-covering in KL.} This divergence exhibits strong mode-covering tendencies as in \cref{fig:cartoon-cover}. Examining the definition of the KL divergence, we see that there is a significant penalty for failing to completely support the demonstration distribution, but no explicit penalty for generating outlier samples. In fact, if $\exists s, a \textup{ s.t. }  \distState{\expert}(s, a) > 0, \distState{\policy}(s, a) = 0$, then the divergence is infinite. However, the opposite does not hold. Thus, the \fIL{KL} optimal policy in $\policyClass$ belongs to the second behavior class  -- which the agent to frequently crash into the tree.
\vspace{-1em}
\paragraph{Mode-collapsing in RKL.} At the other end of the multi-modal behavior spectrum lies the RKL divergence, which exhibits strong mode-seeking behavior as in \cref{fig:cartoon-collapse}, due to switching the expectation over $\distState{\policy}$ with $\distState{\expert}$. Note there is no explicit penalty for failing to entirely cover $\distState{\expert}$, but an arbitrarily large penalty for generating samples which would are improbable under the demonstrator distribution. This results in always turning left or always turning right around the tree, depending on the initialization and mode mixture. For many tasks, failing in such a manner is predictable and safe, as we have already seen similar trajectories from the demonstrator.
\vspace{-1em}
\paragraph{Jensen-Shannon.} This divergence may fall into either behavior class, depending on the MDP, the demonstrations, and the optimization initialization. Examining the definition, we see the divergence is symmetric and expectations are taken over both $\distState{\policy}$ and $\distState{\expert}$. Thus, if either distribution is unsupported (i.e.\ $\exists s, a \textup{ s.t. }  \distState{\expert}(s, a) > 0, \distState{\policy}(s, a) = 0$ or vice versa) the divergence remains finite. Later, we empirically show that although it is possible to achieve safe mode-collapse with JS on some tasks, this is not always the case.

%% file: inputs/experiments.tex

\section{Experiments}
\label{sec:experiments}

\subsection{Low dimensional tasks}

In this section, we empirically validate the following \textbf{Hypotheses}:
\begin{itemize}[itemsep=.3pt,topsep=.1pt]
\item[\textbf{H1}] \textit{The globally optimal policy for RKL imitates a subset of the demonstrator modes, whereas JS and KL tend to interpolate between them.}
\item[\textbf{H2}] \textit{The sample-based estimator for KL and JS underestimates the divergence more than RKL.} 
\item[\textbf{H3}] \textit{The policy gradient optimization landscape for KL and JS with continuously parameterized policies is more susceptible to local minima, compared to RKL.}
\end{itemize}

We test these hypothesis on two environments. The \textbf{Bandit environment} has a single state and three actions, $a$, $b$ and $c$. The expert chooses $a$ and $b$ with equal probability as in \figref{fig:bandit_expert}. We choose a policy class $\Pi$ which has $3$ policies $A$, $B$, and $M$. $A$ selects $a$, $B$ selects $b$ and $M$ stochastically selects $a$, $b$, or $c$ with probability $(\epsilon_0,\epsilon_0,1-2\epsilon_0)$. The \textbf{GridWorld environment} has a $3\times3$ states (\figref{fig:gridworld}). There are a start (\texttt{S}) and a terminal (\texttt{T}) state. The center state is undesirable. The environment has control noise $\epsilon_1$ and transition noise $\epsilon_2$. \figref{fig:grid_expert_s} shows the expert's multi-modal demonstration. The policy class $\Pi$ allows agents to go \textit{up, right, down, left} at each state.

\begin{figure}[!htbp]
\vspace{-.7em}
        \begin{subfigure}[b]{0.32\textwidth}
        		\centering
                \includegraphics[width=0.7\linewidth]{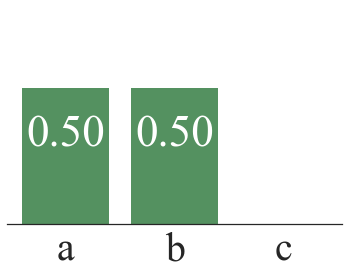}
                \caption{Bandit expert policy}
                \label{fig:bandit_expert}
        \end{subfigure}%
        \begin{subfigure}[b]{0.2\textwidth}
        		\centering
                \includegraphics[height=.8\linewidth]{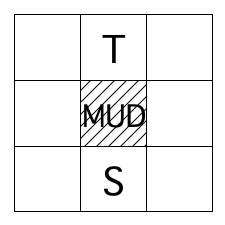}
                \caption{Gridworld}
                \label{fig:gridworld}
        \end{subfigure}%
        \begin{subfigure}[b]{0.2\textwidth}
        		\centering
                \includegraphics[height=.8\linewidth]{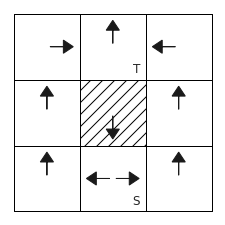}
                \caption{Expert policy}
                \label{fig:grid_expert}
        \end{subfigure}%
        \begin{subfigure}[b]{0.2\textwidth}
        		\centering
                \includegraphics[height=.801\linewidth]{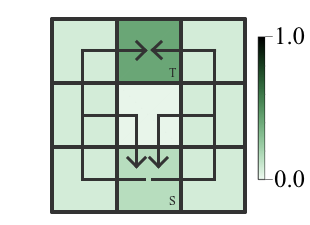}
                \caption{Rollouts}
                \label{fig:grid_expert_s}
        \end{subfigure}
        \caption{Bandit and gridworld environment.}
\end{figure}
\vspace{-2em}
\paragraph{Policy enumeration}
To test \textbf{H1}, we enumerate through all policies in $\Pi$, exactly compute their stationary distributions $\distState{\policy}(\state,\action)$, and select the policy with the smallest exact $f$-divergence, the optimal policy. Our results on the bandit and gridworld ( Table \ref{fig:bandit-a} and \ref{fig:bandit-b}) show that the globally optimal solution to the RKL objective successfully collapses to a single mode (e.g.\ A and Right, respectively), whereas KL and JS interpolate between the modes (i.e.\ M and Up, respectively).


\begin{table}[t]
\setlength{\belowcaptionskip}{0pt}
\setlength{\intextsep}{0pt}
\setlength{\floatsep}{0pt}
\setlength{\abovecaptionskip}{0pt}
\setlength\aboverulesep{0pt}
\caption{Globally optimal policies produced by policy enumeration (\ref{fig:bandit-a} and \ref{fig:bandit-b}), and locally optimal policies produced by policy gradient (\ref{fig:bandit-c} and \ref{fig:bandit-d}). In all cases, the RKL policy tends to collapse to one of the demonstrator modes, whereas the other policies interpolate between the modes, resulting in unsafe behavior.}
\label{tab:grid_pi}
\centering
\begin{tabularx}{\textwidth}
{
    >{\centering\arraybackslash}m{0.33in}
    >{\centering\arraybackslash}m{0.68in}
    >{\centering\arraybackslash}m{0.68in}
    >{\centering\arraybackslash}m{0.68in}@{\hskip .1in}
    >{\centering\arraybackslash}m{0.68in}
    >{\centering\arraybackslash}m{0.68in}
    >{\centering\arraybackslash}m{0.68in}
} \toprule
& \multicolumn{3}{c}{\textbf{Bandit}} & \multicolumn{3}{c}{\textbf{GridWorld}} \\[0.1cm] 
\cmidrule(lr){2-4} \cmidrule(lr){5-7} & RKL & JS & KL & RKL & JS & KL\\
\cmidrule(lr){1-4} \cmidrule(lr){5-7} 
\textbf{H1 }{\tiny global-optima} &
  \multicolumn{3}{l}{
    \begin{subfigure}[t]{.45\textwidth}
		\centering
		\includegraphics[width=0.6in]{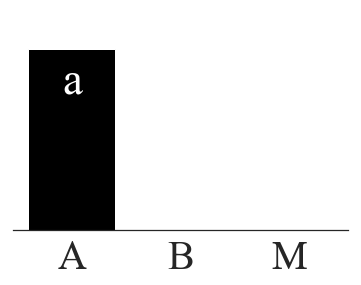}\quad
		\includegraphics[width=0.6in]{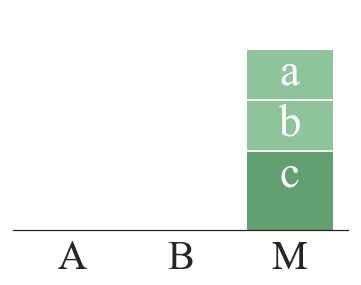}\quad
		\includegraphics[width=0.6in]{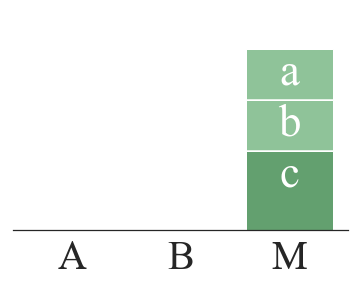}\quad
		\caption{}
		\label{fig:bandit-a} 
	\end{subfigure}
	} & \multicolumn{3}{l}{
	\begin{subfigure}[t]{.45\textwidth}
		\centering
		\includegraphics[width=0.6in]{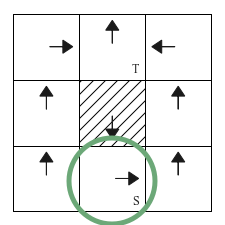}\quad
		\includegraphics[width=0.6in]{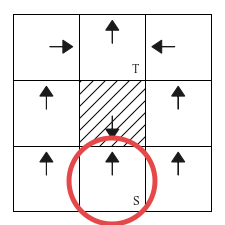}\quad
		\includegraphics[width=0.6in]{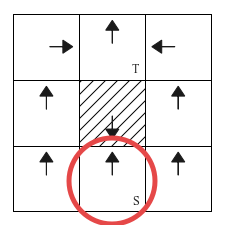}\quad
		\caption{}
		\label{fig:bandit-b}
	\end{subfigure}\quad
	} \\[0.1cm] 
    \textbf{H3 }{\tiny local-optima} &
\multicolumn{3}{l}{
	\begin{subfigure}[t]{.45\textwidth}
	\centering
	\includegraphics[width=0.6in]{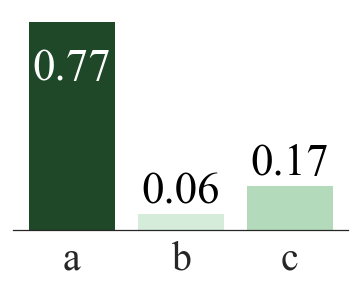}\quad
	\includegraphics[width=0.6in]{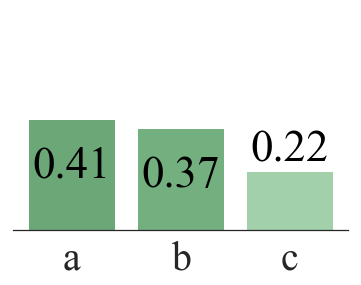}\quad
	\includegraphics[width=0.6in]{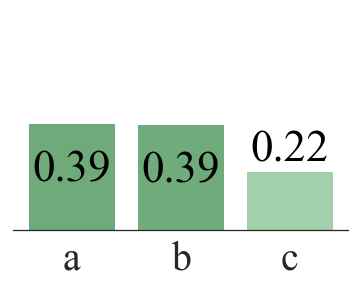}\quad
	\caption{}
	\label{fig:bandit-c}
	\end{subfigure}\quad
  } & \multicolumn{3}{l}{
	\begin{subfigure}[t]{.45\textwidth}
	\centering
	\includegraphics[width=0.6in]{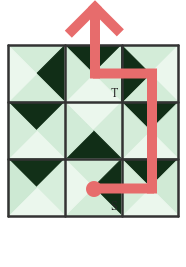}\quad
	\includegraphics[width=0.6in]{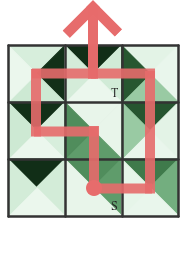}\quad
	\includegraphics[width=0.6in]{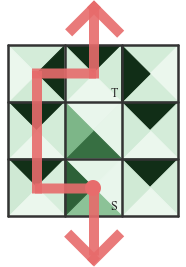}\quad
	\caption{}
	\label{fig:bandit-d}
	\end{subfigure}\quad
  } \\ \bottomrule\end{tabularx}
\label{tab:bandit_pi}
\end{table}

Whether the optimal policy is mode-covering or collapsing depends on the \emph{stochasticity in the policy}. In the bandit environment we parameterize this by $\epsilon_0$ and show in Fig~\ref{fig:control_noise} how the divergences and resulting optimal policy changes as a function of $\epsilon_0$. Note that RKL strongly prefers mode collapsing, KL strongly prefers mode covering, and JS is between the two other divergences.

\begin{figure}[!htbp]
\captionsetup[subfigure]{aboveskip=-1pt,belowskip=-1pt}
\begin{subfigure}{.32\linewidth}
\centering
\includegraphics[width=1.6in]{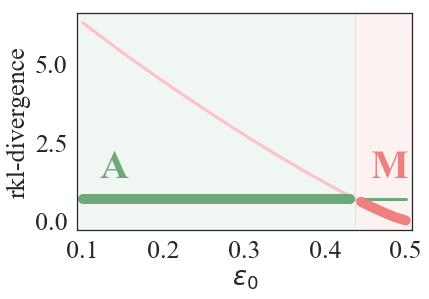}
\caption{RKL}
\label{fig:bandit-cn1}
\end{subfigure}%
\begin{subfigure}{.32\linewidth}
\centering
\includegraphics[width=1.6in]{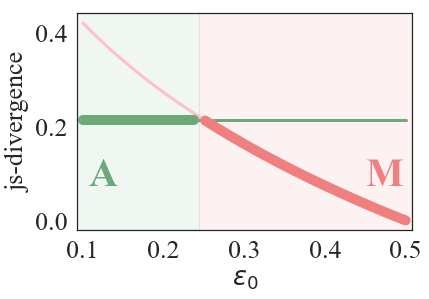}
\caption{JS}
\label{fig:bandit-cn2}
\end{subfigure}
\begin{subfigure}{.32\linewidth}
\centering
\includegraphics[width=1.6in]{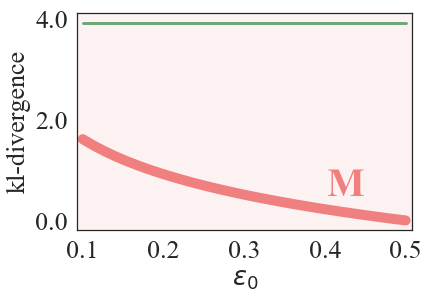}
\caption{KL}
\label{fig:bandit-cn3}
\end{subfigure}
\caption{Divergences and corresponding optimal policy as a function of the control noise $\epsilon_0$. RKL strongly prefers the mode collapse policy $A$ (except at high control noise), KL strongly prefers the mode covering policy $M$, and JS is between the two.}
\label{fig:control_noise}
\end{figure}

\vspace{-.5em}
\paragraph{Divergence estimation}
To test \textbf{H2}, we compare the sample-based estimation of $f$-divergence to the true value in \cref{fig:true_versus_est_divergence}. We highlight the preferred policies under each objective (in the 1 percentile of estimations). For the highlighted group, the estimation is often much lower than the true divergence for KL and JS, perhaps due to the sampling issue discussed in Appendix \ref{app:dividezero}.

\vspace{-.5em}
\paragraph{Policy gradient optimization landscape}
To test \textbf{H3}, we solve for a local-optimal policy using policy gradient for \fIL{KL}, \fIL{RKL} and \fIL{JS}. Though the bandit problem and the gridworld environment have only discrete actions, we consider a continuously parameterized policy class (Appendix \ref{app:parameterize}) for use with policy gradient. Table \ref{fig:bandit-c} and \ref{fig:bandit-d} shows that RKL-VIM empirically produces policies that collapses to a single mode whereas JS and KL-VIM do not. 


\begin{figure}[!t]
\captionsetup[subfigure]{aboveskip=-1pt,belowskip=-1pt}
\begin{subfigure}{.32\linewidth}
\centering
\includegraphics[width=1.2in]{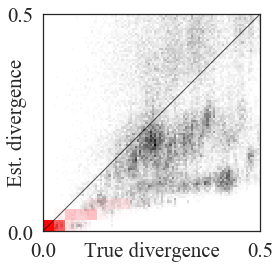}
\caption{RKL}
\label{fig:bandit-cn1}
\end{subfigure}%
\begin{subfigure}{.32\linewidth}
\centering
\includegraphics[width=1.2in]{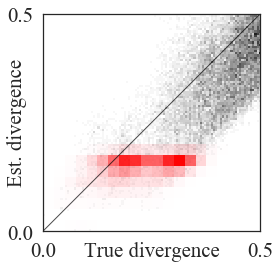}
\caption{JS}
\label{fig:bandit-cn2}
\end{subfigure}
\begin{subfigure}{.32\linewidth}
\centering
\includegraphics[width=1.2in]{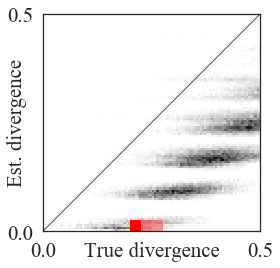}
\caption{KL}
\label{fig:bandit-cn3}
\end{subfigure}
\caption{Comparing $f$-divergence with the estimated values. 
Preferred policies under each objective (in the 1 percentile of estimations) are in red. The normalized estimations appear to be typically lower than the normalized true values for JS and KL.}\label{fig:true_versus_est_divergence}
\end{figure}

\subsection{High dimensional continuous control task}

\begin{figure}[!t]
 \centering
 \includegraphics[width=0.95\linewidth]{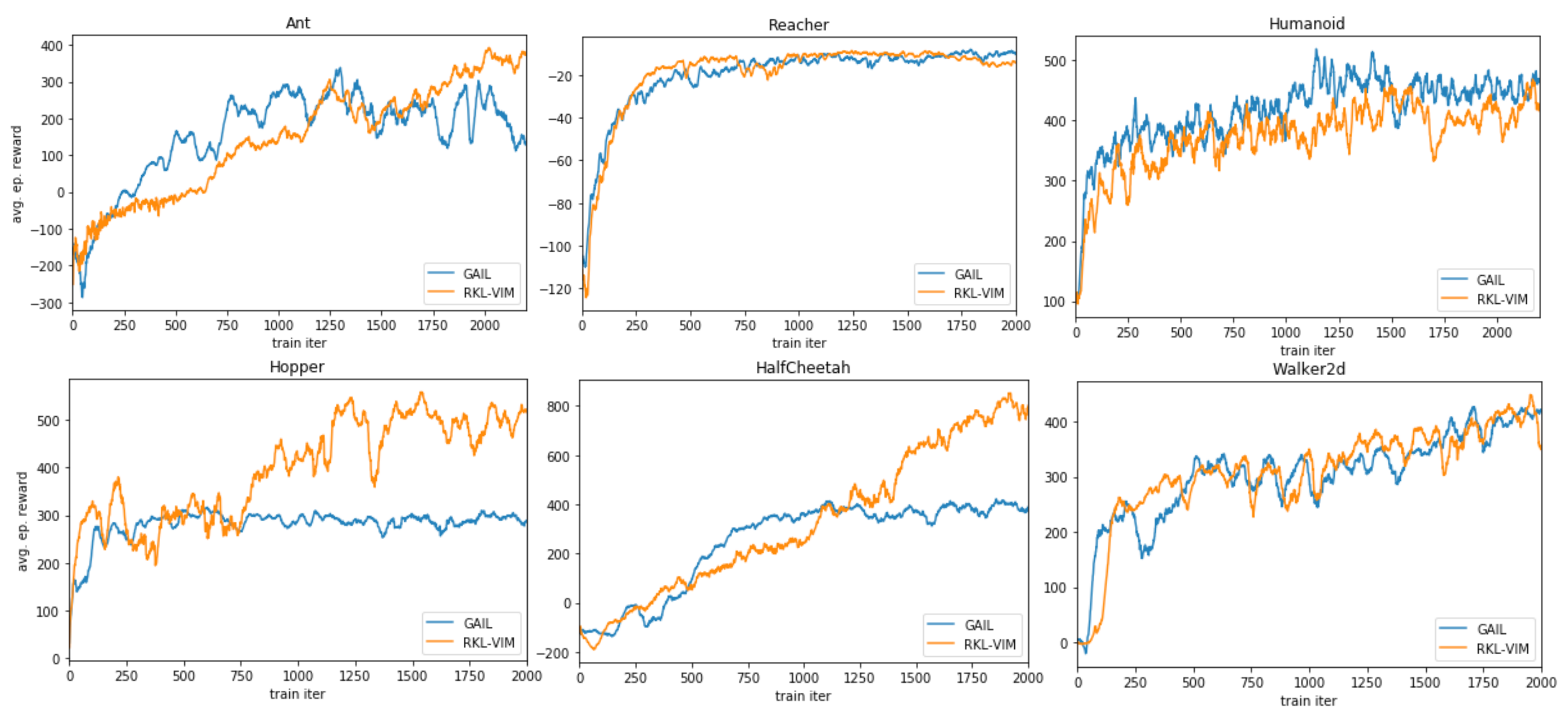}
 \caption{Training \fIL{RKL} and \fIL{JS}(GAIL) on Mujoco environments.}
\label{fig:mujoco}
\end{figure}

We tested \fIL{RKL} and \fIL{JS} (GAIL) on a set of high dimensional control tasks in Mujoco. 
Though our main interest is in multi-modal behavior which occurs frequently in human demonstrations,
here we had to generate expert demonstrations using a reinforcement learning policy, which are \emph{single modal}.

The vanilla version of these algorithms were significantly slow to maximize the cumulative reward. Further examination revealed that there were multiple saddle points that the system gets `stuck' in. A reliable way to coax the algorithm to the desired saddle point was to mix in a small percentage of the true reward along with the discriminator loss. Hence, we augmented the generator loss $\expect{ \pair{\state}{\action} \sim \distTraj{\policy} }{ (1-\alpha) -f^*(g_f(V_w(\state, \action))) + \alpha r(\state, \action)}$ where $\alpha=0.2$. This resulted in reliable, albeit different, convergence from both algorithms.  

\figref{fig:mujoco} shows the average episodic reward over training iterations. On Humanoid, Reacher and Walker2d the performance of both algorithms are similar. However on Ant, Hopper and HalfCheetah \fIL{RKL} converges to a higher value. Further inspection of the discriminator loss reveals that \fIL{RKL} heavily upweights states that the expert visits over the states that the learner visits. While this makes convergence slightly more sluggish (e.g. Ant), the algorithm terminates with a higher reward.  

%% file: inputs/discussion.tex

\section{Discussion}
\label{sec:discussion}
We presented an imitation learning framework based on $f$-divergences, which generalizes existing approaches including behavior cloning (KL), GAIL (JS), and \Dagger (TV). In settings with multi-modal demonstrations, we showed that RKL divergence safely and efficiently collapses to a subset of the modes, whereas KL and JS often produce unsafe behavior. 

Our framework minimizes an \emph{approximate estimation} of the divergence, notably a lower bound \eref{eq:prob:fdiv_lb}. KL divergence is the only one we can actually measure (Appendix~\ref{sec:existing_algorithms}). The lower bound \eref{eq:prob:fdiv_lb} is tight if the function approximator $\phi(x)$ has enough capacity to express the function $f'(\frac{p(x)}{q(x)})$. For Reverse KL, $f(u) = - \log u$ and  $f'(u) = - \frac{1}{u}$. Hence $f'(.)$ can be unbounded and we may need exponentially large number of samples to correctly estimate $\phi(x)$. On the other hand, deriving a \emph{tight upper bound} on the $f$-divergence from a finite set of samples is also impossible. e.g. For RKL, without any assumptions about the expert or learner distribution, there is no way to estimate the support accurately given a finite number of samples. Hence we are left only with the choice of $\infty$ which is vacuous. 

There are a few practical remedies that center around a key observation --  we care not about measuring the divergence but rather minimizing it. One way to do so is to consider a \emph{noisy} version of divergence minimization as in \cite{zhang2019variational}, essentially adding Gaussian noise to both learner and expert to ensure both distributions are absolutely continuous. This upper bounds the magnitude of the divergence. We can think of this as smoothing out the cost function that the policy chooses to minimize. This would help in faster convergence.

We can take these intuitions further and view imitation learning as computing a really good loss - a balance between a loss that maximizes likelihood of expert actions (KL divergence) and a loss that penalizes the learner from visiting states that the expert does not visit. Instead of using estimating the latter term, we can potentially exploit side information. For example, we may already know that the expert does not like to violate obstacle constraints (a fact that we can test from the data). This can then be simply added in as an auxiliary penalty term.

There are a couple interesting directions for future work. One is to unify this framework with maximum entropy moment matching. Given a set of basis function $\phi(x)$, MaxEnt solves for a maximum entropy distribution $q(x)$ such that the moments of the basis functions are matched $\expect{x \sim p(x)}{\phi(x)} = \expect{x \sim q(x)}{\phi(x)}$. Contrast this to \eref{eq:prob:fdiv_lb} where moments of a transformed function are matched. Consequently, MaxEnt \emph{symmetrically} bumps down cost of expert states and bumps up the cost of learner states. In contrast, RKL-VIM \eref{eq:rklvim} \emph{exponentially} bumps down cost of expert and \emph{linearly} bumps up the cost of learner states.

Another interesting direction would be to consider the class of integral probability metrics (IPM). IPMs are metrics that take the form $\sup_{\phi \in \Phi} \expect{x \sim p(x)}{\phi(x)} - \expect{x \sim q(x)}{\phi(x)}$. Unlike f-divergence estimators, these metrics are measurable by definition. Choosing different families of $\Phi$ results in MMD, TotalVariation, Earth-movers distance. Preliminary results using such estimators seem promising~\cite{sun2019provably}. 


%% file: inputs/acknowledge.tex
\paragraph{Acknowledgements} This work was (partially) funded by the National Institute of Health R01 (\#R01EB019335), National Science Foundation CPS (\#1544797), National Science Foundation NRI (\#1637748), the Office of Naval Research, the RCTA, Amazon, and Honda Research Institute USA.

%% file: appendix/state_action_distribution.tex

\section{Lower Bounding $f$-Divergence of Trajectory Distribution with State Action Distribution}
\label{sec:state_action_distribution}

We begin with a lemma that relates $f$-divergence between two vectors and their sum.
\begin{lemma}[Generalized log sum inequality]
\label{lem:generalized_log_sum_inequality}
Let $p_1, \dots, p_n$ and $q_1, \dots, q_n$ be non-negative numbers. Let $p = \sum_{i=1}^n p_i$ and $q = \sum_{i=1}^n q_i$. Let $f(.)$ be a convex function. We have the following:
\begin{equation}
	\sum_{i=1}^n  q_i f \left( \frac{p_i}{q_i} \right) \geq q f \left( \frac{p}{q} \vphantom{\frac{p_i}{q_i}} \right)
\end{equation}
\end{lemma}

\begin{proof}
\begin{align}
\sum_{i=1}^n  q_i f \left( \frac{p_i}{q_i} \right) 	&= q \sum_{i=1}^n  \frac{q_i}{q} f \left( \frac{p_i}{q_i} \right) 
\geq q f \left( \sum_{i=1}^n  \frac{q_i}{q} \frac{p_i}{q_i} \right) \label{eq:log_sum:convex}
\geq q f \left( \frac{1}{q} \sum_{i=1}^n  p_i \vphantom{\frac{p_i}{q_i}} \right) 
\geq q f \left( \frac{p}{q} \vphantom{\frac{p_i}{q_i}} \right) 
\end{align}
where \eref{eq:log_sum:convex} is due to Jensen's inequality since $f(.)$ is convex and $q_i \geq 0$ and $\sum_{i=1}^n \frac{q_i}{q} = 1$. 
\end{proof}

We use \lemref{lem:generalized_log_sum_inequality} to prove a more general lemma that relates the f-divergence defined over two spaces where one of the space is rich enough in information to explain away the other.
\begin{lemma}[Information loss]
\label{lem:fdivergence_reduction}
Let $a$ and $b$ be two random variables. Let $P(a,b)$ be a joint probability distribution. The marginal distributions are $P(a) = \sum_b P(a,b)$ and $P(b) = \sum_a P(a,b)$. 
Assume that $a$ \textbf{contains all information of} $b$. This is expressed as follows -- given any two probability distribution $P(.)$, $Q(.)$, assume the following equality holds for all $a, b$:
\begin{equation}
\label{eq:constant_condition}
	P(b | a) = Q(b | a)
\end{equation}
Under these conditions, the following inequality holds:
\begin{equation}
	\sum_{a } Q(a) f \bigg( \frac{P(a)}{Q(a)} \bigg) \geq \sum_{b} Q(b) f \bigg( \frac{P(b)}{Q(b)} \bigg)
\end{equation}
\end{lemma}

\begin{proof}

\begin{align}
\sum_{a } Q(a) f \bigg( \frac{P(a)}{Q(a)} \bigg)	&= \sum_{a}\bigg(\sum_{b} Q(a,b)\bigg) f\bigg( \frac{P(a)}{Q(a)} \bigg)\\
&= \sum_{a} \sum_{b} Q(a,b) f\bigg( \frac{P(a)}{Q(a)} \bigg) \\
&= \sum_{b} \sum_{a} Q(a,b) f\bigg( \frac{ P(a,b) / P(b | a)} { Q(a,b) / Q(b | a)} \bigg)\label{eq:joint_formula} \\
&= \sum_{b} \sum_{a} Q(a,b) f\bigg( \frac{ P(a,b)} { Q(a,b)} \bigg) \label{eq:apply_constant_condition}\\
&\geq \sum_{b} \bigg( \sum_{a} Q(a,b) \bigg) f\left( \frac{ \bigg( \sum_{a} P(a,b) \bigg)} { \bigg( \sum_{a} Q(a,b) \bigg) } \right) \label{eq:apply_log_sum} \\
&\geq \sum_{b} Q(b) f \bigg( \frac{P(b)}{Q(b)} \bigg)
\end{align}

We get \eref{eq:joint_formula} by applying $P(a,b) = P(a)P(b | a)$ and $Q(a,b) = Q(a)Q(b | a)$. We get \eref{eq:apply_constant_condition} applying the equality constraint from \eref{eq:constant_condition}. We get \eref{eq:apply_log_sum} from \lemref{lem:generalized_log_sum_inequality} by setting $p_i = P(a,b), q_i = Q(a,b)$ and summing over all $a$ keeping $b$ fixed. 
\end{proof}

We are now ready to prove \thmref{thm:state_action_distribution} using \lemref{lem:fdivergence_reduction}.

\begin{proof}[Proof of \thmref{thm:state_action_distribution}]

Let random variable $a$ belong to the space of trajectories $\traj$. Let random variable $b$ belong to the space of state action pairs $z = (\state, \action)$. Note that for any joint distribution $P(z, \traj)$ and $Q(z, \traj)$, the following is true
\begin{equation}
	P(z|\traj) = Q(z| \traj)
\end{equation}
This is because a trajectory $\traj$ contains all information about $z$, i.e. $\traj = \{s_0, a_1, s_1, \dots \}$. Upon applying \lemref{lem:fdivergence_reduction} we have the inequality
\begin{equation}
\sum_{\tau } Q(\tau) f\frac{P(\tau)}{Q(\tau)} \geq \sum_{z} Q(z) f\frac{P(z)}{Q(z)} = \sum_{(s,a)} \rho_\pi(s,a) f\big(\frac{\rho_{\pi^*}(s,a)}{\rho_\pi(s,a)}\big)
\end{equation}
\end{proof}

The bound is reasonable as it merely states that information gets lost when temporal information is discarded. Note that the theorem also extends to state distributions, i.e.
\begin{corollary}
Divergence between trajectory distribution is lower bounded by state distribution.
\begin{equation*}
\fDiv{\distTraj{\expert}(\traj)}{\distTraj{\policy}(\traj)} \geq \fDiv{\distTraj{\expert}(\state)}{\distTraj{\policy}(\state)}
\end{equation*}
\end{corollary}

How tight is this lower bound? We examine the gap 
\begin{corollary}
The gap between the two divergences is
\begin{align*}
\fDiv{P(\traj)}{Q(\traj)} - \fDiv{P(z)}{Q(z)}  = \sum_z P(z) \fDiv{P(\traj|z)}{Q(\traj|z)}
\end{align*}
\end{corollary}

\begin{proof}
\begin{align*}
\sum_{\tau} P(z,\tau) f \bigg( \frac{P(z,\tau)}{Q(z,\tau)} \bigg) - P(z) f \bigg( \frac{P(z)}{Q(z)} \bigg) &= P(z)\sum_{\tau} P(\tau|z) f \bigg( \frac{P(\tau|z)P(z)}{Q(\tau|z)Q(z)} \bigg) - P(z) f \bigg( \frac{P(z)}{Q(z)} \bigg)\\
&= P(z) \bigg( \sum_{\tau} P(\tau|z) f \bigg( \frac{P(\tau|z)P(z)}{Q(\tau|z)Q(z)} \bigg) - f \bigg( \frac{P(z)}{Q(z)}\bigg) \bigg) \\
&= P(z) \bigg( \sum_{\tau} P(\tau|z) f \bigg( \frac{P(\tau|z)P(z)}{Q(\tau|z)Q(z)} \bigg) - \sum_{\tau}P(\tau|z) f \bigg( \frac{P(z)}{Q(z)}\bigg) \bigg) \\
&= P(z) \cdot \fDiv{P(\tau|z)}{Q(\tau|z)}
\end{align*}
where we use $\sum_{\tau}P(\tau|z) = 1$. 
\end{proof}
Let $\mathcal{A}$ be the set of trajectories that contain $z$, i.e., $\mathcal{A} = \{\tau | P(z|\tau) = Q(z|\tau) > 0\}$. The gap is the conditional f-divergence of $\tau \in \mathcal{A}$ scaled by $P(z)$. The gap comes from whether we treat $\tau \in \mathcal{A}$ as separate events (in the case of trajectories) or as the same event (in the case of $z$).

%% file: appendix/upper_action_bound.tex

\section{Relating $f$-Divergence of Trajectory Distribution with Expected Action Distribution}
\label{sec:action_distribution}

In this section we explore the relation of divergences between induced trajectory distribution and induced action distribution. We begin with a general lemma

\begin{lemma} 
\label{lem:frequency_counts}
Given a policy $\policy$ and a general feature function $\phi(\state,\action)$, the expected feature counts along induced trajectories is the same as expected feature counts on induced state action distribution
\begin{equation}
\sum_{\traj} \distTraj{\policy}(\traj) \left( \sum_{t=1}^T \phi(\state_{t-1}, \action_t) \right) = T \sum_\state \distState{\policy}(\state) \sum_\action \policy(a | s) \phi(s,a)
\end{equation}
\end{lemma}

\begin{proof}
Expanding the LHS we have
\begin{align}
&\sum_{\traj} \distTraj{\policy}(\traj) \left( \sum_{t=1}^T \phi(\state_{t-1}, \action_t) \right) \\
=&\sum_{t=1}^T  \sum_{\traj} \distTraj{\policy} \phi(\state_{t-1}, \action_t) \\
=&\sum_{t=1}^T  \sum_{\state_{t-1}} \distStateTime{\policy}{t-1}(\state_{t-1}) \sum_{\action_t} \policy(\action_t | \state_{t-1}) \phi(\state_{t-1}, \action_t) \sum_{\state_{t+1}} P(\state_{t} | \state_{t-1}, \action_t) \dots \sum_{\state_T} P(\state_T | \state_{T-1}, \action_T) \\
=&\sum_{t=1}^T  \sum_{\state_{t-1}} \distStateTime{\policy}{t-1}(\state_{t-1}) \sum_{\action_t} \policy(\action_t | \state_{t-1}) \phi(\state_{t-1}, \action_t) \label{eq:avg_stat:marginalize} \\
=&\sum_{t=1}^T  \sum_{\state} \distStateTime{\policy}{t-1}(\state) \sum_{\action} \policy(\action | \state) \phi(\state, \action) \label{eq:avg_stat:notime}\\
=&T  \sum_{\state} \distState{\policy}(\state) \sum_{\action} \policy(\action | \state) \phi(\state, \action) \label{eq:avg_stat:avg_dist}
\end{align}	
where \eref{eq:avg_stat:marginalize} is due to marginalizing out the future, \eref{eq:avg_stat:notime} is due to the fact that state space is same across time and \eref{eq:avg_stat:avg_dist} results from applying average state distribution definition.
\end{proof}

We can use this lemma to get several useful equalities such as the average state visitation frequency
\begin{theorem} Given a policy $\policy$, if we do a tally count of states visited by induced trajectories, we recover the average state visitation frequency. 
\label{thm:average_state_dist}
\begin{equation}
\sum_{\traj} \distTraj{\policy}(\traj) \left( \sum_{t=1}^T \frac{1}{T} \Ind(\state_{t-1} = z) \right) = \distState{\policy}(z)
\end{equation}
\end{theorem}

\begin{proof}
Apply \lemref{lem:frequency_counts} with $\phi(\state,\action)= \Ind(\state = z)$
\end{proof}

Unfortunately \lemref{lem:frequency_counts} does not hold for $f$-divergences in general. But we can analyze a subclass of $f$-divergences that satisfy the following triangle inequality:
\begin{equation}
\label{eq:fdiv:triangle_inequality}
 \fDiv{p}{q} \leq \fDiv{p}{r} + \fDiv{r}{q}
\end{equation}
Examples of such divergences are Total Variation distance or Squared Hellinger distance. 

We now show that for such divergences (which are actually distances), we can \emph{upper bound} the $f$-divergence. Contrast this to the \emph{lower bound}
discussed in \appref{sec:state_action_distribution}. The upper bound is attractive because the trajectory divergence is the term that we actually care about bounding.

Also note the implications of the upper bound -- we now need expert labels on states collected by the learner $\state \sim \distState{\policy}$. Hence we need an interactive expert that we can query from arbitrary states. 
\begin{theorem}[Upper bound]
\label{thm:action_distribution}
Given two policies $\pi$ and $\expert$, and f-divergences that satisfy the triangle inequality, divergence between the trajectory distribution is upper bounded by the expected divergence between the action distribution on states induced by $\pi$.
\begin{equation*}
\fDiv{\distTraj{\expert}(\traj)}{\distTraj{\policy}(\traj)} \leq T \expect{\state \sim \distState{\policy}}{\fDiv{\expert(\action | \state)}{\policy(\action | \state)}}
\end{equation*}
\end{theorem}

\begin{proof}[Proof of \thmref{thm:action_distribution}]

We will introduce some notations to aid in explaining the proof. Let a trajectory segment $\trajSeg{t}{\hor}$ be 
\begin{equation}
	\trajSeg{t}{\hor} = \{ \state_t, \action_{t+1}, \state_{t+1}, \dots, \action_\hor, \state_\hor\}
\end{equation}

Recall that the probability of a trajectory induced by policy $\policy$ is
\begin{equation}
	\distTraj{\policy}(\trajSeg{0}{\hor}) = \initDist(\state_0) \prod \limits_{t=1}^{\hor} \policy(\action_t | \state_{t-1}) \transFn{\state_{t-1}}{\action_t}{\state_{t}}
\end{equation}

We also introduce a non-stationary policy $\policyBar$ that executes $\policy$ and then $\expert$ thereafter. Hence, the probability of a trajectory induced by $\policyBar$ is
\begin{equation}
	\distTraj{\policyBar}(\trajSeg{0}{\hor}) = \initDist(\state_0) \policy(\action_1 | \state_0) \transFn{\state_{0}}{\action_1}{\state_{1}} \prod \limits_{t=2}^{\hor} \expert(\action_t | \state_{t-1}) \transFn{\state_{t-1}}{\action_t}{\state_{t}}
\end{equation}

Let us consider the divergence between distributions $\distTraj{\expert}(\trajSeg{0}{\hor})$ and $\distTraj{\policy}(\trajSeg{0}{\hor})$ and apply the triangle inequality \eref{eq:fdiv:triangle_inequality} with respect to $\distTraj{\policyBar}(\trajSeg{0}{\hor})$
\begin{align}
& \fDiv{ \distTraj{\expert}(\trajSeg{0}{\hor}) }{ \distTraj{\policy}(\trajSeg{0}{\hor}) } \\
\leq &  \fDiv{ \distTraj{\expert}(\trajSeg{0}{\hor}) }{ \distTraj{\policyBar}(\trajSeg{0}{\hor}) } + \fDiv{ \distTraj{\policyBar}(\trajSeg{0}{\hor}) }{ \distTraj{\policy}(\trajSeg{0}{\hor}) } \\
\leq &  \sum_{\trajSeg{0}{\hor}} \distTraj{\policyBar}(\trajSeg{0}{\hor}) f \bigg( \frac{ \distTraj{\expert}(\trajSeg{0}{\hor}) }{ \distTraj{\policyBar}(\trajSeg{0}{\hor}) } \bigg)+ \fDiv{ \distTraj{\policyBar}(\trajSeg{0}{\hor}) }{ \distTraj{\policy}(\trajSeg{0}{\hor}) } \\
\leq &  \sum_{\trajSeg{0}{\hor}} \distTraj{\policyBar}(\trajSeg{0}{\hor}) f \left( \frac{ \initDist(\state_0) \expert(\action_1 | \state_0) \transFn{\state_{0}}{\action_1}{\state_{1}} \prod \limits_{t=2}^{\hor} \expert(\action_t | \state_{t-1}) \transFn{\state_{t-1}}{\action_t}{\state_{t}} }{ \initDist(\state_0) \policy(\action_1 | \state_0) \transFn{\state_{0}}{\action_1}{\state_{1}} \prod \limits_{t=2}^{\hor} \expert(\action_t | \state_{t-1}) \transFn{\state_{t-1}}{\action_t}{\state_{t}} } \right) \\
& + \fDiv{ \distTraj{\policyBar}(\trajSeg{0}{\hor}) }{ \distTraj{\policy}(\trajSeg{0}{\hor}) } \nonumber \\
\leq &  \sum_{\trajSeg{0}{\hor}} \distTraj{\policyBar}(\trajSeg{0}{\hor}) f \left( \frac{ \expert(\action_1 | \state_0) }{ \policy(\action_1 | \state_0) }\right) + \fDiv{ \distTraj{\policyBar}(\trajSeg{0}{\hor}) }{ \distTraj{\policy}(\trajSeg{0}{\hor}) } \\
\leq &  \sum_{\state_0} \initDist(\state_0) \sum_{\action_1} \policy(\action_1 | \state_0 ) f \left( \frac{ \expert(\action_1 | \state_0) }{ \policy(\action_1 | \state_0) }\right) \sum_{\state_1} \transFn{\state_{0}}{\action_1}{\state_{1}} \sum_{\action_1} \dots  + \fDiv{ \distTraj{\policyBar}(\trajSeg{0}{\hor}) }{ \distTraj{\policy}(\trajSeg{0}{\hor}) } \\
\leq &  \sum_{\state_0} \initDist(\state_0) \sum_{\action_1} \policy(\action_1 | \state_0 ) f \left( \frac{ \expert(\action_1 | \state_0) }{ \policy(\action_1 | \state_0) }\right)  + \fDiv{ \distTraj{\policyBar}(\trajSeg{0}{\hor}) }{ \distTraj{\policy}(\trajSeg{0}{\hor}) } \\
\leq &  \sum_{\state} \initDist(\state) \sum_{\action} \policy(\action | \state ) f \left( \frac{ \expert(\action | \state) }{ \policy(\action | \state) }\right)  + \fDiv{ \distTraj{\policyBar}(\trajSeg{0}{\hor}) }{ \distTraj{\policy}(\trajSeg{0}{\hor}) } \\
\leq & \expect{\state \sim \initDist(\state)}{ \fDiv{\expert(\action|\state)}{\policy(\action|\state)} } + \fDiv{ \distTraj{\policyBar}(\trajSeg{0}{\hor}) }{ \distTraj{\policy}(\trajSeg{0}{\hor}) } \label{eq:upper_bound_first_eq}
\end{align}

Expanding the second term we have 
\begin{align}
& \fDiv{ \distTraj{\policyBar}(\trajSeg{0}{\hor}) }{ \distTraj{\policy}(\trajSeg{0}{\hor}) } \\
=& \sum_{\trajSeg{0}{\hor}} \distTraj{\policy}(\trajSeg{0}{\hor}) f \left( \frac{\initDist(\state_0) \policy(\action_1 | \state_0)\transFn{\state_{0}}{\action_1}{\state_{1}} \prod \limits_{t=2}^{\hor} \expert(\action_t | \state_{t-1}) \transFn{\state_{t-1}}{\action_t}{\state_{t}}}{\initDist(\state_0) \policy(\action_1 | \state_0) \transFn{\state_{0}}{\action_1}{\state_{1}} \prod \limits_{t=2}^{\hor} \policy(\action_t | \state_{t-1}) \transFn{\state_{t-1}}{\action_t}{\state_{t}}} \right)\\
=& \sum_{\state_0} \initDist(\state_0) \sum_{\action_1} \policy(\action_1 | \state_0 ) \sum_{\state_1} P(\state_1 | \state_0, \action_1) \dots 
f \left( \frac{\transFn{\state_{0}}{\action_1}{\state_{1}} \prod \limits_{t=2}^{\hor} \expert(\action_t | \state_{t-1}) \transFn{\state_{t-1}}{\action_t}{\state_{t}}}{\transFn{\state_{0}}{\action_1}{\state_{1}} \prod \limits_{t=2}^{\hor} \policy(\action_t | \state_{t-1}) \transFn{\state_{t-1}}{\action_t}{\state_{t}}} \right) \\
=& \sum_{\state_0} \initDist(\state_0) \sum_{\action_1} \policy(\action_1 | \state_0 ) \sum_{\trajSeg{1}{\hor}} \distTraj{\policy}(\trajSeg{1}{\hor}) f \left( \frac{\distTraj{\expert}(\trajSeg{1}{\hor})}{\distTraj{\policy}(\trajSeg{1}{\hor})} \right) \\
=& \sum_{\state_0} \initDist(\state_0) \sum_{\action_1} \policy(\action_1 | \state_0 ) \fDiv{ \distTraj{\expert}(\trajSeg{1}{\hor}) }{ \distTraj{\policy}(\trajSeg{1}{\hor}) } 
\end{align}

We can apply triangle inequality again with respect to $\distTraj{\policyBar}(\trajSeg{1}{\hor})$ to get
\begin{align}
&\sum_{\state_0} \initDist(\state_0) \sum_{\action_1} \policy(\action_1 | \state_0 ) \fDiv{ \distTraj{\expert}(\trajSeg{1}{\hor}) }{ \distTraj{\policy}(\trajSeg{1}{\hor}) } \\
\leq & \sum_{\state_0} \initDist(\state_0) \sum_{\action_1} \policy(\action_1 | \state_0 ) \left[  \fDiv{ \distTraj{\expert}(\trajSeg{1}{\hor}) }{ \distTraj{\policyBar}(\trajSeg{1}{\hor}) } +  \fDiv{ \distTraj{\policyBar}(\trajSeg{1}{\hor}) }{ \distTraj{\policy}(\trajSeg{1}{\hor}) } \right]  \\
\leq & \sum_{\state_0} \initDist(\state_0) \sum_{\action_1} \policy(\action_1 | \state_0 ) \left[ \sum_{\state_1} \transFn{\state_{0}}{\action_1}{\state_{1}} \sum_{\action_2} \policy(\action_2 | \state_1) f \left( \frac{\expert(\action_2 | \state_1)}{\policy(\action_2 | \state_1)} \right)  +  \fDiv{ \distTraj{\policyBar}(\trajSeg{1}{\hor}) }{ \distTraj{\policy}(\trajSeg{1}{\hor}) } \right]  \\
\leq & \sum_{\state} \distStateTime{\policy}{0}(\state) \sum_\action \policy(\action | \state) f \left( \frac{\expert(\action | \state)}{\policy(\action | \state)} \right) + \sum_{\state_0} \initDist(\state_0) \sum_{\action_1} \policy(\action_1 | \state_0 ) \fDiv{ \distTraj{\policyBar}(\trajSeg{1}{\hor}) }{ \distTraj{\policy}(\trajSeg{1}{\hor}) } \\
\leq & \expect{\state \sim \distStateTime{\policy}{0}(\state)}{ \fDiv{\expert(\action|\state)}{\policy(\action|\state)} }  \\ 
& +  \sum_{\state_0} \initDist(\state_0) \sum_{\action_1} \policy(\action_1 | \state_0 ) \sum_{\state_1} \transFn{\state_{0}}{\action_1}{\state_{1}} \sum_{\action_2} \policy(\action_2 | \state_1 ) \fDiv{ \distTraj{\expert}(\trajSeg{2}{\hor}) }{ \distTraj{\policy}(\trajSeg{2}{\hor}) } \nonumber 
\end{align}
Again if we continue to expand $\fDiv{ \distTraj{\expert}(\trajSeg{2}{\hor}) }{ \distTraj{\policy}(\trajSeg{2}{\hor}) }$ and add to \eref{eq:upper_bound_first_eq} we have
\begin{align}
\fDiv{ \distTraj{\expert}(\trajSeg{0}{\hor}) }{ \distTraj{\policy}(\trajSeg{0}{\hor}) }  
&\leq \sum_{t=0}^{\hor-1} \expect{\state \sim \distStateTime{\policy}{t}(\state)}{ \fDiv{\expert(\action|\state)}{\policy(\action|\state)} } \\
&\leq \hor \expect{\state \sim \distState{\policy}(\state)}{ \fDiv{\expert(\action|\state)}{\policy(\action|\state)} } \label{eq:upper_bound_final}
\end{align}

where \eref{eq:upper_bound_final} follows from $\distState{\policy}(\state) = \frac{1}{\hor} \sum_{t=0}^{\hor-1} \distStateTime{\policy}{t}(\state)$

\end{proof}

%% file: appendix/existing_algorithms.tex

\section{Existing algorithms as different f-divergence minimization}
\label{sec:existing_algorithms}

\paragraph{Behavior Cloning -- Kullback-Leibler (KL) divergence.}

If we use KL divergence $f(u) = u \log (u)$ in our framework's trajectory matching problem:

\begin{equation} 
\label{eq:bc}
\begin{aligned}
D_{KL}(\rho_{\pi^*}(\tau), \rho_{\pi}(\tau)) &= \sum_\tau \rho_{\pi^*}(\tau) \log ( \frac{\rho_{\pi^*}}{\rho_{\pi}} ) = \sum_\tau \rho_{\pi^*}(\tau) \log ( \prod_t \frac{\pi^*(a_t | s_{t-1})}{\pi(a_t | s_{t-1})} ) \\
& = \sum_\tau \rho_{\pi^*}(\tau) \sum_t \log (\frac{\pi^*(a_t | s_{t-1})}{\pi(a_t | s_{t-1})}) \\
& = \mathbb{E}_{s\sim\rho_{\pi^*}, a\sim \pi^*} [\log \pi^*(a|s) - \log \pi(a|s)] \\
\policyLearn & = \min D_{KL}(\rho_{\pi^*}(\tau), \rho_{\pi}(\tau)) \\ 
& = \max \mathbb{E}_{\state \sim \distState{\expert}, \action \sim \expert(\cdot |\state)} \log(\policy(\action | \state))\\
\end{aligned}
\end{equation}

Note that this is exactly the behavior cloning~\citep{Pomerleau88} objective, which tries to minimize a classification loss under the expert's state-action distribution. The loss used in \eref{eq:bc} is the cross entropy loss for multi-class classification. This optimization is also referred to as \emph{M-projection}. A benefit of this method is that it does not rely on any interactions with the environment; data is provided by the expert.

It's well known that behavior cloning often leads to covariant shift problem in practice~\citep{ross2011reduction}. One explanation is that supervised learning errors compound exponentially in time. We can also view this a side-effect of M-projection which can lead to situations where $\policy(a|\state) > 0$, $\expert(a | \state) = 0$. 

\paragraph{Generative Adversarial Imitation Learning (GAIL)~\citep{ho2016generative} -- Jensen-Shannon (JS) divergence.}

Plugging in the JS divergence $f(u) = -(u+1)\log \frac{1+u}{2} + u \log u$ in \eref{eq:fil} we have

\begin{equation}
\label{eq:gail}
\policyLearn = \argminprob{\policy \in \policyClass} \max_{w} \; \expect{ \pair{\state}{\action} \sim \distTraj{\expert} }{\log D_w(\state, \action)} - \expect{ \pair{\state}{\action} \sim \distTraj{\policy} }{ \log (1-D_w(\state, \action))} 
\end{equation}
this matches the GAIL objective (without the entropic regularizer). Note that this is minimizing an estimate of the lower bound of JS divergence. While this requires a more expensive minimax optimization procedure, but at least in practice GAIL appears to outperform behavior cloning on a range of simulated environments. 

\paragraph{Dataset Aggregation (\Dagger)~\citep{ross2011reduction} -- Total Variation (TV) distance.}

If we choose $f\left( u \right) = \frac{1}{2} \abs{u - 1} $ in \eref{eq:prob:fdiv_def}, we get the total variation distance $\tvDiv{p}{q} = \frac{1}{2} \sum_x \abs{p(x) - q(x)} $. TV satisfies the triangle inequalities and hence can be shown to satisfy the following:

\begin{theorem} The Total Variation distance between trajectory distributions is upper bounded by the expected distance between the action distribution on states induced by $\pi$.
\begin{equation}
	\tvDiv{ \distTraj{\expert}(\traj) }{ \distTraj{\policy}(\traj) } \leq \hor \expect{\state \sim \distState{\policy}(\state)}{ \tvDiv{\expert(\action|\state)}{\policy(\action|\state)} } \leq \hor \sqrt{ \expect{\state \sim \distState{\policy}(\state)}{ \klDiv{\expert(\action|\state)}{\policy(\action|\state)} } }
\end{equation}
\end{theorem}
\begin{proof}
We first apply \thmref{thm:action_distribution} on total variation distance. Then by Cauchy-Schwartz inequality we have
\begin{equation}
	\expect{\state \sim \distState{\policy}(\state)}{ \tvDiv{\expert(\action|\state)}{\policy(\action|\state)} } \leq \sqrt{ \expect{\state \sim \distState{\policy}(\state)}{ \left( \tvDiv{\expert(\action|\state)}{\policy(\action|\state)} \right)^2 } }
\end{equation}

Finally by Pinsker's inequality we have
\begin{equation}
	\left( \tvDiv{\expert(\action|\state)}{\policy(\action|\state)} \right)^2 \leq \klDiv{\expert(\action|\state)}{\policy(\action|\state)}
\end{equation}
Putting all inequalities together we have the proof. 
\end{proof}

\Dagger solves the  following non i.i.d learning problem
\begin{equation}
\begin{aligned}
\policyLearn &= \argminprob{\pi \in \Pi} \; \expect{\state \sim \distState{\policy}(\state), \action \sim \expert(\action|\state)}{\ell(s,a)}\\
&= \argminprob{\pi \in \Pi} \; \expect{\state \sim \distState{\policy}(\state), \action \sim \expert(\action|\state)}{ - \log \left( \policy(\action|\state) \right)} \\
&= \argminprob{\pi \in \Pi} \; \expect{\state \sim \distState{\policy}(\state), \action \sim \expert(\action|\state)}{ \klDiv{\expert(\action|\state)}{\policy(\action|\state)}} \\
\end{aligned}
\end{equation}

\Dagger reduces this to an iterative supervised learning problem. Every iteration a classification algorithm is called. Let $\epsilon_N = \min_{\pi \in \Pi} \frac{1}{N} \sum_{i=1}^N \expect{\state \sim \distState{\policy_i}(\state)}{ \klDiv{\expert(\action|\state)}{\policy(\action|\state)} }$. Let $\gamma_N$ be the average regret which goes to zero asymptotically, i.e. $\lim_{N \rightarrow \infty} \gamma_N = 0$. \Dagger guarantees that there exists a learnt policy $\policy$ that satisfies the following bound (infinite sample case):
\begin{equation}
\begin{aligned}
\expect{\state \sim \distState{\policy}(\state)}{ \klDiv{\expert(\action|\state)}{\policy(\action|\state)} } \leq \epsilon_N + \gamma_N + \bigo{\frac{\log T}{N}}
\end{aligned}
\end{equation}

Putting all together we have a bound on total variation distance $T\sqrt{ \epsilon_N + \gamma_N + \bigo{\frac{\log T}{N}} }$.

\begin{figure}[!t]
    \centering 
    \includegraphics[width=0.8\textwidth]{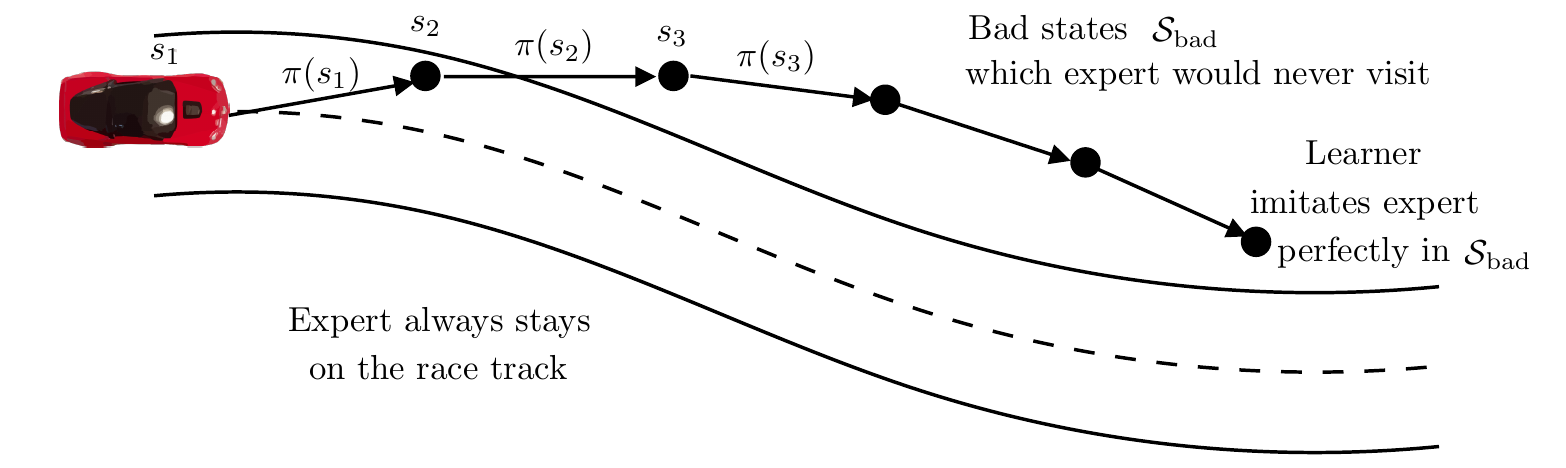}
    \caption{%
    Problem with the \Dagger formulation. The expert policy stays perfectly on the race track. The learner immediately goes off the race track and visits bad states.
    \label{fig:problem_dagger}
}
\end{figure}
We highlight some undesirable behaviors in \Dagger. Consider the example shown in \figref{fig:problem_dagger}. The expert stays on the track and never visits bad states $\state \in \badStateSpace$. The learner, on the other hand, immediately drifts off into $\badStateSpace$. Moreover, for all $\state \in \badStateSpace$, the learner can perfectly imitate the expert. In other words, $\lossFn(\state, \policy) = 0$ for these states. In fact, it is likely that for certain policy classes, this is the optimal solution! At the very least, \Dagger is susceptible to learn such policies as is shown in the counter example in \cite{laskey2017comparing}.

This phenomenon can also be explained from the lens of Total Variation distance. TV measures $\tvDiv{p}{q} = \frac{1}{2} \sum_x \abs{p(x) - q(x)}$. This distance does not penalize the learner going off the track as much as RKL. In this case, RKL would have a very high penalization for reasons mentioned in \sref{sec:divergence}.

%% file: appendix/rkl_action_distribution.tex

\section{Reverse KL Divergence via Interactive Variational Imitation}
\label{sec:rkl_action_distribution}

\fIL{RKL} is an approximation of I-projection that minimizes the lower bound of Reverse KL divergence. There are few things we don't like about it. First, it's a double lower bound -- moving to state-action divergence (\thmref{thm:state_action_distribution}) and then variational lower bound \eref{eq:prob:fdiv_lb}. Second, the optimal state action divergence estimator may require a complex function class. Finally, the min-max optimization \eref{eq:fil} may be slow to converge. Interestingly, Reverse KL has a special structure that we can exploit to do even better if we have an interactive expert!

\begin{theorem}
\begin{equation}
\rklDiv{\distTraj{\expert}(\traj)}{\distTraj{\policy}(\traj)} = T  \sum_{\state} \distState{\policy}(\state) \rklDiv{\expert(\action | \state)}{\policy(\action | \state)}
\end{equation}

which means

\begin{equation}
\sum_{\traj} \distTraj{\policy}(\traj) \log \left( \frac{\distTraj{\policy}(\traj)}{\distTraj{\expert}(\traj)} \right) = T  \sum_{\state} \distState{\policy}(\state) \sum_\action \policy(a | s) \log \left( \frac{\policy(a | s)}{\expert(a | s)} \right)
\end{equation}

\begin{proof}
Applying lemma \ref{lem:frequency_counts} with $\phi(\state,\action)=\log \left( \frac{\policy(a | s)}{\expert(a | s)} \right)$
\end{proof}
\end{theorem}

Note that different from Theorem~\ref{thm:action_distribution}, we have strict equality in the above equation. Hence we can directly minimize action distribution divergence. Since this is on states induced by $\policy$, this falls under the regime of \emph{interactive learning}~\citep{ross2011reduction}. In this regime, we need to query the expert on \emph{states visited by the learner}. Note that this may not always be convenient - the expert is required during training. However, as we will see, it does lead to better estimators. 

We now explore variational imitation approaches similar to \fIL{RKL} for  minimizing action divergence. We get the following update rule:
\begin{equation}
\begin{aligned}
\label{eq:rkl:action}
\policyLearn &= \argminprob{\policy \in \policyClass} \; \expect{\state \sim \distState{\policy}} { \expect{ \action \sim \expert(. | \state) }{- \exp(V_w(\state, \action))} + \expect{  \action \sim \policy(. | \state) }{ V_w(\state, \action)} } \\
\end{aligned}
\end{equation}

We call this \emph{interactive variational imitation} (\ifIL{RKL}). The algorithm is described in \algoref{alg:iRKL}. Unlike \fIL{RKL}, we collect a fresh batch of data from \emph{both} the expert and the learner every iteration. The optimal estimator in this case is $V_{w^*}(\state, \action) = \log \left( \frac{\policy(\action | \state)}{\expert(\action | \state)} \right)$. This can be far simpler to estimate than the optimal estimator for state-action divergence $V_{w^*}(\state, \action) = \log \left( \frac{ \distState{\policy}(\state)\policy(\action | \state) }{ \distState{\expert}(\state)\expert(\action | \state) } \right)$. 

\begin{algorithm}[!htbp]
\caption{\ifIL{RKL} \label{alg:iRKL}}
\begin{algorithmic}[1]
\State Initialize learner and estimator parameters $\theta_0$, $w_0$
\For{$i=0$ \textbf{to} $N-1$}
\State Sample state-action pairs from learner $\tau_i \sim \distTraj{\policy_{\theta_i}}$
\State Query expert on all trajectories $\tau_i$ to get $\action^* \sim \expert( . | \state)$.
\State Update estimator $w_{i+1} \leftarrow w_{i} + \eta_w \nabla_w \left( \expect{\state \sim \tau_i} { \expect{ \action \sim \action^* }{- \exp(V_w(\state, \action))} + \expect{ \action \sim \tau_i }{ V_w(\state, \action)} }\right)$
\State Update policy (using policy gradient) $\theta_{i+1} \gets \theta_i - \eta_\theta \; \expect{ \pair{\state}{\action} \sim \tau_i }{ \nabla_\theta \log \policy_\theta (\action | \state) Q^{V_w}(\state, \action)} )$
\Statex \hspace{3ex} where $Q^{V_w}(\state_{t-1}, \action_t) = \sum\limits_{i=t}^T V_w(\state_{i-1}, \action_i)$
\EndFor
\State \textbf{Return} $\policy_{\theta_N}$ 
\end{algorithmic}
\end{algorithm}

%% file: appendix/dre.tex

\section{Density Ratio Minimization for Reverse KL via No Regret Online Learning}
\label{app:i_RKL}

We continue the argument made in \ref{sec:rkl_action_distribution} for better algorithms for reverse KL minimization. Instead of variational lower bound, we can \emph{upper bound} the action divergence as follows:

\begin{equation}
\begin{aligned}
\label{eq:dre_opt}
  \frac{1}{T}\rklDiv{\distTraj{\expert}(\traj)}{\distTraj{\policy}(\traj)} = \expect{\state \sim \distState{\policy}}{ \expect{\action \sim \policy(.|\state)}{ \log \frac{\policy(\action | \state)}{\expert(\action | \state)}}} \leq \expect{\state \sim \distState{\policy}}{ \expect{\action \sim \policy(.|\state)}{  \frac{ \policy(\action | \state) }{ \expert(\action | \state) } - 1 }} \notag\\
\end{aligned}
\end{equation} where we use the fact that $\log(x) \leq x$ for any $x >0$. To estimate the conditional density ratio $\pi(a|s)/\pi^*(a|s)$, we can leverage an off-shelf density ratio estimator (DRE) \footnote{Given two distributions $p(z)\in\Delta(\mathcal{Z})$ and $q(z) \in \Delta(\mathcal{Z})$ over a finite set $\mathcal{Z}$, the Density Ratio Estimator (DRE) aims to compute an estimator $\hat{r}: \mathcal{Z} \to \mathbb{R}^{+}$ such that, $\hat{r}(z) \approx p(z) / q(z)$ (we assume $q$ has no smaller support than $p$ on $\mathcal{Z}$, i.e., $q(z) = 0$ implies $p(z) = 0$.), with access only to two sets of samples $\{z_i\}_{i=1}^N \sim p$ and $\{z'_i\}_{i=1}^N \sim q$. In this work, we treat DRE as a black box that takes two datasets as input, and returns a corresponding ratio estimator: $\hat{r} = DRE(\{z_i\}_{i=1}^N, \{z'_i\}_{i=1}^N)$. We further assume that DRE achieves small prediction error: $\mathbb{E}_{z\sim q }\left[ |\hat{r}(z) - p(z) / q(z) | \right] \leq \delta \in \mathbb{R}^+$.
} as follows. Rather then directly estimating $\pi(a|s)/\pi^*(a|s)$ for all $s$, we notice that $\pi(a|s)/\pi^*(a|s) = \left(\rho_{\pi}(s)\pi(a|s)\right)/\left(\rho_{\pi}(s)\pi^*(a|s)\right)$. We know how to sample $(s,a)$ from $\rho_{\pi}(s)\pi(a|s)$, and under the interactive setting, we can also sample $(s,a)$ from $\rho_{\pi}(s)\pi^*(a|s)$ by first sampling $s\sim \rho_{\pi}(\cdot)$ and then sample action $a\sim \pi^*(\cdot |s)$, i.e., query expert at state $s$. Given a dataset $D = \{s,a\} \sim \rho_{\pi}\pi$, and a dataset $D^* = \{s,a^*\}\sim \rho_{\pi}\pi^*$, DRE takes the two datasets and returns an estimator: $\hat{r} = DRE(D, D^*)$ such that $\hat{r}(s,a) \approx \frac{\rho_{\pi}(s)\pi(a|s)}{\rho_{\pi}(s)\pi^*(a|s)} = \frac{\pi(a|s)}{\pi^*(a|s)}$. Hence, by just using a classic DRE, we can form a conditional density ratio estimator via leveraging the interactive expert.  

With the above trick to estimate conditional density ratio estimator, now we are ready to introduce our algorithm (Alg.~\ref{alg:i_RKL}). Our algorithm takes a density ratio estimator (DRE) and a cost-sensitive classifier as input.\footnote{A cost sensitive classifier $\mathcal{C}$ takes a dataset $\{s, c\}$ with $c\in \mathbb{R}^{|\mathcal{A}|}$ as input, and outputs a classifier that minimizes the classification cost: $\pi =\arg\min_{\pi\in\Pi} \sum_{i} c_{\pi(s)}$, where we use $c_{a}$ to denote the entry in the cost vector $c$ that corresponds to action $a$.} At the $n$-th iteration, it uses the current policy $\pi_n$ to generate states $s$, and then collect a dataset $\{s, a\}$ with $a\sim \pi_n(\cdot|s)$, and another dataset $\{s,a^*\}$ with $a^* \sim \pi^*(\cdot|s)$. It then uses DRE to learn a conditional density ratio estimator $\hat{r}_n (s,a) \approx \pi_n(a|s) / \pi^*(a|s)$ for any $(s,a)\in\mathcal{S}\times\mathcal{A}$. It then performs data aggregation by aggregating newly generated state cost vector pairs $\{s, \hat{r}_n(\cdot,s)\}$ to the cost-sensitive classification dataset $\mathcal{D}$. We then update the policy to $\pi_{n+1}$ via performing cost-sensitive classification on $\mathcal{D}$.

\begin{algorithm}[!htbp]
\caption{Interactive DRE Minimization \label{alg:i_RKL}}
\begin{algorithmic}[1] 
\State \textbf{Input}: Density Ratio Estimator (DRE) $\mathcal{R}$, expert $\pi^*$, Cost sensitive classifier $\mathcal{C}$
\State Initialize learner $\pi_0$, dataset $\mathcal{D} = \emptyset$
\For{$n=0$ \textbf{to} $N-1$}
	\State $s_0 \sim \rho_0$
	\State Initialize $D = \emptyset$, $D^* = \emptyset$ 
	\For {$e = 0$ \textbf{to} $E$}
		\For {$t = 0$ \textbf{to} ${T-1}$}
			\State Query expert: $a^*_t \sim \pi^*(\cdot| s_t) $
			\State Execute action $a_t \sim \pi_n(\cdot|s_t)$ and receive $s_{t+1}$
			\State  $D = D \cup \{s_t, a_t\}$, $D^* = D^*\cup \{s_t,a_t^*\}$ \label{Line:dre_data}
		\EndFor
	\EndFor
	\State $\hat{r}_n  = \mathcal{R}(D, D^*)$ \Comment{Density Ratio Estimation}
	\State $\mathcal{D} = \mathcal{D}\cup \{s, \hat{r}_n(\cdot, s )\}_{(s,a)\in D}$ \Comment{Data Aggregation}
	\State $\pi_{n+1} = \mathcal{C}(\mathcal{D})$  \Comment{Cost sensitive classification}
	\EndFor
\State \textbf{Return} $\policy_{\theta_N}$ 
\end{algorithmic}
\end{algorithm}



Below we provide an agnostic analysis of the performance of the returned policy from Alg.~\ref{alg:i_RKL}.  Our analysis is reduction based, in a sense that the performance of the learned policy depends on performance of the off-shelf DRE, the performance of the cost sensitive classifier , and the no-regret learning rate.  
Similar to DAgger \cite{ross2011reduction}, note that Alg.~\ref{alg:i_RKL} can be understood as running Follow-The-Leader (FTL)  no-regret online learner on the sequence of loss functions $\hat{\ell}_n(\pi) \triangleq \mathbb{E}_{s\sim \rho_{\pi_n}} \left[\mathbb{E}_{a\sim \pi(\cdot|s)} \left[\hat{r}_n(a,s)-1\right]  \right] $ for $n \in [0, N-1]$, which approximate $\ell_n(\pi) \triangleq \mathbb{E}_{s\sim \rho_{\pi_n}}[\mathbb{E}_{a\sim \pi(\cdot|s)} [r_n(s,a) - 1]]$. We denote $\epsilon_{\text{class}} = \min_{\pi\in \Pi} \frac{1}{N}\sum_{i=0}^{N-1}{\ell}_n(\pi)$ as the minimal cost sensitive classification error one could achieve in hindsight.  Note the $\epsilon_{\text{class}}$ represents the richness of our policy class $\Pi$. If $\Pi$ is rich enough such that $\pi^*\in\Pi$, we have $\epsilon_{\text{class}} = 0$. In general, $\epsilon_{\text{class}}$ decreases when we increase the  representation power of policy class $\Pi$.
Without loss of generality, we also assume the following black-box DRE oracle $\mathcal{R}$ performance:
\begin{align}
\label{eq:dre_guarantee}
\max_{n\in [N]} \mathbb{E}_{s\sim \rho_{\pi_n}} \left[ \mathbb{E}_{a\sim \pi^*(\cdot|s)} \lvert \hat{r}_n(s,a) -  r_n(s,a) \rvert \right] \leq \gamma, 
\end{align} with $r_n(s,a) = \rho_{\pi_n}(s)\pi_n(a|s) / (\rho_{\pi_n}(s)\pi^*_n(a|s))$ being the true ratio.  Note that this is the standard performance guarantee one can get from the theoretical foundation of Density Ratio Estimation. In Appendix~\ref{app:DRE_example} we give an example of DRE with its finite sample performance guarantee in the form of Eq.~\ref{eq:dre_guarantee}.  We also assume that the expert has non-zero probability of trying any action at any state, i.e., $\min_{s,a}\pi^*(a|s) \geq c \in [0,1]$.  We ignore sample complexity here and simply focus on analyzing the quality of the learned policy under the assumption that every round we can draw enough samples to accurately estimate all expectations. Finite sample analysis can be done via standard concentration inequalities. 


\begin{theorem} Run Alg.~\ref{alg:i_RKL} for $N$ iterations. 
Then there exists a policy $\pi\in \{\pi_0, \dots, \pi_{N-1}\}$ such that
\begin{align}
D_{RKL}(\rho_{\pi^*}, \rho_{\pi}) 
\leq T \left(\left(1+\frac{1}{c}\right)\gamma + \epsilon_{\text{class}} + \frac{o(N)}{N}\right). \nonumber
\end{align} 
\label{thm:i_RKL}
\end{theorem}
The detailed proof is deferred to Appendix~\ref{app:i_RKL}. By definition of $o(N)$, we have $\lim_{N\to\infty}o(N)/N \to 0$.  The above theorem indicates that as $N\to \infty$, the inverse KL divergence is upper bounded by $(1+1/c)\gamma + \epsilon_{\text{class}}$. Increasing the representation power of $\Pi$ will decrease $\epsilon_{\text{class}}$ and any performance improvement the density ratio estimation community could make on DRE can be immediately transferred to a performance improvement of Alg.~\ref{alg:i_RKL}.

\subsection{Proof of Theorem~\ref{thm:i_RKL}}
Denote the ideal loss function at iteration $n$ as 
\begin{align}
\label{eq:dagger_idea_loss}
\ell_{n}(\pi) = \mathbb{E}_{s\sim \rho_{\pi_n}}\left[\mathbb{E}_{a\sim \pi(\cdot|s)} [r_n(s,a)] \right],
\end{align} where $r_n(s,a) = \frac{\pi_n(a|s)}{\pi^*(a|s)} - 1$ (note we include the constant $-1$ for analysis simplicity).  Note that we basically treat $r_n(s, a)$ as a cost of $\pi$ classifying to action $a$ at state $s$.  

Of course, we do not have a perfect $r_n(s,a)$, as we cannot access $\pi^*$'s likelihood. Instead we rely on an off-shelf density ratio estimation (DRE) solver to approximate $r_n$ by $\hat{r}_n$. We simply assume that the returned $\hat{r}_n$ has the following performance guarantee:
\begin{align}
\mathbb{E}_{s\sim \rho_{\pi_i}}\left[\mathbb{E}_{s\sim \pi^*(\cdot|s)} |\hat{r}_i(s,a) - r_i(s,a)| \right] \leq \gamma. 
\end{align} Note that this performance guarantee is well analyzed and is directly delivered by existing off-shelf density ratio estimation algorithms (e.g., \cite{nguyen2010estimating,kanamori2012statistical}). Such $\gamma$ in general depends on the richness of the function approximator we use to approximate $r$, and the number of samples we draw from $\rho_{\pi_n}\pi_n$ and $\rho_{\pi_n}\pi^*$, and can be analyzed using classic learning theory tools. In realizable setting (i.e., our hypothesis class contains $r(s,a)$) and in the limit where we have infinitely many samples, $\gamma$ will be zero. The authors in \cite{nguyen2010estimating} analysis $\gamma$ with respect to the number of samples under the realizable assumption.

With $\hat{r}_n$, let us define $\hat{\ell}_n(\pi)$ that approximates $\ell_n(\pi)$ as follows:
\begin{align}
\hat{\ell}_n(\pi) = \mathbb{E}_{s\sim \rho_{\pi_n}}\left[\mathbb{E}_{a\sim \pi(\cdot|s)} [\hat{r}_n(s,a)] \right],
\end{align} where we simply replace $r_n$ by $\hat{r}_n$. Now we bound the difference between $\ell_n(\pi^*)$ and $\hat{\ell}_n(\pi^*)$ using $\gamma$ (the reason we use $\pi^*$ inside $\ell$ and $\hat{\ell}$ will be clear later):
\begin{align}
\label{eq:relation_ell_hat_ell_under_expert}
|\ell_n(\pi^*) - \hat{\ell}_n(\pi^*) | &= |\mathbb{E}_{s\sim \rho_{\pi_n}}\mathbb{E}_{a\sim \pi^*(\cdot|s)} (r_n(s,a) - \hat{r}_n(s,a)) | \nonumber\\
& \leq \mathbb{E}_{s\sim \rho_{\pi_n}}\mathbb{E}_{a\sim \pi^*(\cdot|s)} | r_n(s,a) - \hat{r}_n(s,a)| \leq \gamma,
\end{align} where we simply applied Jenson inequality. 

Note that at this stage, we can see that the Alg.~\ref{alg:i_RKL} is simply using FTL on a sequence of loss functions $\ell_n(\pi)$ for $n\in [N]$.  The no-regret property from FTL immediately gives us the following inequality:
\begin{align}
\sum_{i=1}^N \hat{\ell}_i(\pi_i) - \min_{\pi\in\Pi} \sum_{i=1}^N \hat{\ell_i}(\pi) \leq o(N).
\end{align}
Let us examine the second term on the LHS of the above inequality: $ \min_{\pi\in\Pi} \sum_{i=1}^N \hat{\ell_i}(\pi)$: 
\begin{align}
 \min_{\pi\in\Pi} \frac{1}{N}\sum_{i=1}^N \hat{\ell_i}(\pi) \leq  \frac{1}{N}\sum_{i=1}^N \hat{\ell_i}(\pi^*)  \leq \frac{1}{N}\sum_{i=1}^N \ell_i(\pi^*)  + \gamma = \epsilon_{\text{class}} + \gamma,
\end{align} where we used inequality~\ref{eq:relation_ell_hat_ell_under_expert}, and the fact that $\ell_i(\pi^*) = \mathbb{E}_{s\sim d_{\pi_i}}\mathbb{E}_{a\sim \pi^*(\cdot|s)} r_i(s,a) = 0$ (recall we define $r_i(s,a) = \pi_i(a|s)/\pi^*(a|s) - 1$).  Hence, we get there exists a least one policy $\hat{\pi}$ among $\{\pi_i\}_{i=1}^N$, such that:
\begin{align}
\mathbb{E}_{s\sim \rho_{\hat{\pi}}}\left[\mathbb{E}_{a\sim \hat{\pi}(\cdot|s)} [{r}(s,a)] \right] \leq \gamma + o(N)/N,
\end{align} where we denote ${r}(s,a)$ as the ratio approximator of $\hat{\pi}(a|s)/\pi^*(a|s)$. 

We now link $\hat{\ell}_i(\pi_i)$ and $\ell_i(\pi_i)$ as follows:
\begin{align}
&|\hat{\ell}_i(\pi_i) -  \ell_i(\pi_i) | \leq \mathbb{E}_{s\sim \rho_{\pi_i}}\mathbb{E}_{a\sim \pi_i} |\hat{r}_i(s,a) - r_i(s,a)| \nonumber\\
& \leq \left(\max_{s,a} \frac{\pi_i(a|s)}{\pi^*(a|s)}\right) \mathbb{E}_{s\sim \rho_{\pi_i}}  \mathbb{E}_{a\sim \pi*(a|s)} | \hat{r}_i(s,a) - r_i(s,a)|  \leq  \left(\max_{s,a} \frac{\pi_i(a|s)}{\pi^*(a|s)}\right) \gamma \leq \frac{1}{c} \gamma, 
\end{align}where we assumed $\min_{s,a} \pi^*(a|s) \geq c$, which is a necessary assumption to make $D_{KL}(\rho_{\pi}, \rho_{\pi^*})$ well defined. 

Put everything together, we get:
\begin{align}
D_{RKL}(\rho_{\hat{\pi}}, \rho_{\pi^*}) \leq T \mathbb{E}_{s\sim d_{\hat{\pi}}} \left[\mathbb{E}_{a\sim \hat{\pi}(\cdot|s)} [r(s,a)]\right] \leq T\left(\left(1+\frac{1}{c}\right)\gamma +\epsilon_{\text{class}} + o(N)/N\right).
\end{align} Note the linear dependency on $T$ is not improvable and shows up in the original DAgger as well.





 \subsection{An Example of Density Ratio Estimation and Its Finite Sample Analysis}
\label{app:DRE_example}

We consider the algorithm proposed by \citet{nguyen2010estimating} for density ratio estimation. We also provide finite sample analysis, which is original missing in \cite{nguyen2010estimating}.

 We consider the following general density ratio estimation problem. Given a finite set of elements $\mathcal{Z}$, and two probability distribution over $\mathcal{Z}$, $p \in \Delta(\mathcal{Z})$ and $q\in\Delta(\mathcal{Z})$. We are interested in estimating the ratio $r(z) = p(z) / q(z)$.  The first assumption we use in this section is that $q(z)$ is lower bounded by a constant for any $z$:
 \begin{assum}[Boundness]
 \label{ass:boundness}
 There exits a small positive constant $c\in\mathbb{R}^+$, such that for any $z\in\mathcal{Z}$, we always have $q(z) \geq c$ and $p(z)\geq c$.
 \end{assum} Essentially we assume that $q$ has full support over $\mathcal{Z}$. The above assumption ensures that the ratio is well defined for all $z\in\mathcal{Z}$, and $p(z)/q(z)  \in [ c, 1/c]$. 

 Let us assume that we are equipped with a set of function approximators $\mathcal{G} = \{g: \mathcal{Z} \to  [-1/c, 1/c] \}$ with $a, b\in \mathbb{R}^+$ two positive constants. The second assumption is the realizable assumption:
 \begin{assum}[Realizability]
 We assume that $r(z)\triangleq p(z)/q(z) \in \mathcal{G}$, and $\mathcal{G}$ is discrete, i.e., $\mathcal{G}$ contains finitely many function approximators. 
 \label{ass:realizable}
 \end{assum} Note that for analysis simplicity we assume $\mathcal{G}$ is discrete. As we will show later that $|\mathcal{G}|$ is only going to appear inside a log term, hence $\mathcal{G}$ could contain large number of function approximators. Also, our analysis below uses standard uniform convergence analysis with standard concentration inequality (i.e., Hoeffding's inequality), it is standard to relax the above assumption to continuous $\mathcal{G}$ where $\log(|\mathcal{G}|)$ will be replace by complexity terms such as Rademacher complexity.

 Given two sets of samples, $\{x_i\}_{i=1}^N \sim q$, and $\{y_i\}_{i=1}^N \sim p$, we perform the following optimization:
 \begin{align}
 \label{eq:specific_DRE}
 \hat{g} = \arg\min_{g\in\mathcal{G}} \frac{1}{N}\sum_{i=1}^N (g(x_i))^2 - \frac{2}{N}\sum_{i=1}^N g(y_i).
 \end{align}
 One example is that when $\mathcal{G}$ belongs to some Reproducing Kernel Hilbert Space $\mathcal{H}$ with some kernel $k:\mathcal{Z}\times\mathcal{Z} \to \mathbb{R}^+$, then the above optimization has closed form solution. 

 We are interested in bounding the following risk:
 \begin{align}
 \label{eq:dre_risk}
 \mathbb{E}_{z\sim q } |{\hat{g}(z) - r(z)}|.  
 \end{align}

\begin{theorem}
\label{thm:dre_finite_analysis}
 Fix $\delta\in (0,1)$. Under assumption~\ref{ass:boundness} and assumption~\ref{ass:realizable}, with probability at least $1-\delta$, DRE in Eq~\ref{eq:specific_DRE} returns an ratio estimator, such that:
 \begin{align*}
 \mathbb{E}_{z\sim q} |{\hat{g}(z) - r(z)}| \leq \frac{2}{c} \sqrt{\log(2|\mathcal{G}|/\delta)}N^{-1/4}.
 \end{align*}
 \end{theorem}


 Below we prove several useful lemmas that will lead us to Theorem~\ref{thm:dre_finite_analysis}.

 \begin{lemma}
 \label{lem:obj_lower_bound}
 For any $g\in\mathcal{G}$, we have:
 \begin{align*}
 \mathbb{E}_{z\sim q} |g(z) - r(z)| \leq \sqrt{\mathbb{E}_{x\sim q} (g(z) - r(z))^2}.
 \end{align*}
 \end{lemma}
 The proof of the above simply uses Jensen's inequality with the fact that $\sqrt{x}$ is a concave function.  

Define the true risk as $\ell(g)$:
\begin{align*}
\ell(g) = \mathbb{E}_{z \sim q} [g(z)^2 ] - 2\mathbb{E}_{z\sim p}[g(z)].
\end{align*}
Note that by realizability assumption, $r = \arg\min_{g\in\mathcal{G}} \ell(g)$.
For a fixed $g$, denote $v_i(g) = (g(x_i))^2 - 2 g(y_i)$. Note that $\mathbb{E}_{i}[v_i(g)] = \ell(g)$. Also note that $|v_i(g)| \leq 1/c^2$.

\begin{lemma}[Uniform Convergence]
\label{lem:uniform_conv}
With probability at least $1-\delta$, for all $g\in\mathcal{G}$, we have:
\begin{align*}
|\frac{1}{N}\sum_{i=1}^N v_i(g) - \ell(g)| \leq \frac{2}{c^2}\sqrt{\frac{\log(2|\mathcal{G}|/\delta)}{N}}.
\end{align*}
\end{lemma}
The above lemma can be easily proved by first applying Hoeffding's inequality on $\sum_{i=1}^N v_i/N$ for a fixed $g$, then applying union bound over all $g\in\mathcal{G}$.

\begin{lemma}
\label{lem:constant_shift_obj}
For any $g\in\mathcal{G}$, we have:
 \begin{align*}
 \mathbb{E}_{z\sim q} (g(z) - r(z))^2 = \mathbb{E}_{z\sim q} (g(z))^2 - 2\mathbb{E}_{z\sim p}g(z) + \mathbb{E}_{z\sim p} r(z).
 \end{align*}
 \end{lemma}
 \begin{proof}
 From $\mathbb{E}_{z\sim q} (g(z) - r(z))^2$, we complete the square:
 \begin{align*}
 &\mathbb{E}_{z\sim q} (g(z) - r(z))^2 = \mathbb{E}_{z\sim q} g(z)^2 + \mathbb{E}_{z\sim q} r(z)^2 - 2\mathbb{E}_{z\sim q} g(z) r(z) \\
 & = \mathbb{E}_{z\sim q} g(z)^2 + \mathbb{E}_{z\sim p} r(z) - 2 \mathbb{E}_{z\sim p} g(z).
 \end{align*} where we use the fact that that $r(z) = p(z) / q(z)$, and $\mathbb{E}_{z\sim q} r(z)g(z) = \mathbb{E}_{z\sim q} g(z)p(z)/q(z) = \mathbb{E}_{z\sim p} g(z)$.
 \end{proof}

 Now we are ready to prove the main theorem. 
 \begin{proof}[Proof of Theorem~\ref{thm:dre_finite_analysis}]
 We are going to condition on the event that Lemma~\ref{lem:uniform_conv} being hold. Denote $C_N = (2/c^2)\sqrt{\log(2|{\mathcal{G}}|/\delta)/N}$, based on Lemma~\ref{lem:uniform_conv}, we have:
 \begin{align*}
 \ell(\hat{g}) \leq \sum_{i=1}^N v_i(\hat{g}) + C_N \leq \sum_{i=1}^N v_i(r) + C_N \leq \ell(r) + 2 C_N,
 \end{align*} where the first and last inequality uses Lemma~\ref{lem:uniform_conv}, while the second inequality uses the fact that $\hat{g}$ is the minimizer of $\sum_{i=1}^N v_i(g)$, and the fact that $\mathcal{G}$ is realizable.

 Based on Lemma~\ref{lem:constant_shift_obj}, we have:
 \begin{align*}
 \mathbb{E}_{z\sim q}(\hat{g}(z) - r(z))^2 = \ell(\hat{g}) + \mathbb{E}_{z\sim p} r(z) \leq \ell(r) + 2C_N + \mathbb{E}_{z\sim p}r(z) = \mathbb{E}_{z\sim q}(r(z) - r(z))^2 + 2C_N = 2C_N.
 \end{align*}
 Now use Lemma~\ref{lem:obj_lower_bound}, we have:
 \begin{align*}
 \mathbb{E}_{z\sim q} |{\hat{g}(z) - r(z)}| \leq \sqrt{\mathbb{E}_{z\sim q} (\hat{g}(z) - r(z))^2}  \leq \sqrt{2C_N}. 
 \end{align*}Hence we prove the theorem.
 \end{proof}

%% file: appendix/divide_by_zero.tex
\section{Divide-by-zero Issues with KL and JS} \label{app:dividezero}
 Recall the definition of the $f$-divergence
\begin{equation}
 \fDiv{p}{q} = \sum_x q(x) f \left( \frac{p(x)}{q(x)} \right).
\end{equation}
The core divide-by-zero issue is best illustrated by considering the setting where we have samples from $q$ yet can exactly evaluate $p(x)$ and $q(x)$. At first glance, it may appear the following estimator
\begin{equation}
\mathbb{E}_{x \sim q}  f \left( \frac{p(x)}{q(x)} \right).
\end{equation}
is not unreasonable. However, the issue here is if $\exists x \; \textup{s.t.}\; p(x) > 0, q(x) = 0$, then depending on $f$, this divergence may be infinite (in $f$-divergences, $0*\infty = \infty$. Yet the estimator will never sample the location where $q(x) = 0$ and thus fail to realize the infinite divergence. This issue is particularly pronounced in KL and JS due to their respective $f$ functions.

\newpage

\section{Experimental Details} \label{app:parameterize}

\paragraph{Environmental Noise} For the bandit environment we set control noise $\epsilon_0 = 0.28$, unless otherwise specified. For the grid world we tested with the control noise $\epsilon_1=0.14$ and the transitional noise $\epsilon_2=0.15$.

\paragraph{Parameterization of policies}

Both bandit and grid world environment have discrete actions. To transform the discrete action space into a continuous policy space for policy gradient, we consider the following settings. Bandit's policy is parameterized by one real number $\theta$. Given $\theta$, one can construct a vector $V = [\cos(- \theta - \frac{\pi}{4}), \cos(\theta), \cos(- \theta + \frac{\pi}{4})]$. The probability of executing the discrete actions $a,b,c$ is: $\textrm{softmax}[A(V+1)]$ where we set $A=2.5$ in our experiment. For the grid world, the policy is a matrix $\theta$ of size $N$ where $N$ is the number of states. For state $i$ one can construct a vector $V_i = [cos(0-\theta), cos(\pi/2-\theta), cos(\pi-\theta), cos(-\pi/2-\theta)]$. The probability of executing discrete actions \textit{UP, RIGHT, DOWN, LEFT} at state $i$ is: $\textrm{softmax}[A(V_i+1)]$ where we set $A=2.5$.

%% file: author.bbl
\begin{thebibliography}{55}
\providecommand{\natexlab}[1]{#1}
\providecommand{\url}[1]{\texttt{#1}}
\expandafter\ifx\csname urlstyle\endcsname\relax
  \providecommand{\doi}[1]{doi: #1}\else
  \providecommand{\doi}{doi: \begingroup \urlstyle{rm}\Url}\fi

\bibitem[Ross et~al.(2013)Ross, Melik-Barkhudarov, Shankar, Wendel, Dey,
  Bagnell, and Hebert]{ross2013learning}
St{\'e}phane Ross, Narek Melik-Barkhudarov, Kumar~Shaurya Shankar, Andreas
  Wendel, Debadeepta Dey, J~Andrew Bagnell, and Martial Hebert.
\newblock Learning monocular reactive uav control in cluttered natural
  environments.
\newblock In \emph{Robotics and Automation (ICRA), 2013 IEEE International
  Conference on}, 2013.

\bibitem[Finn et~al.(2016{\natexlab{a}})Finn, Levine, and
  Abbeel]{finn2016guided}
Chelsea Finn, Sergey Levine, and Pieter Abbeel.
\newblock Guided cost learning: Deep inverse optimal control via policy
  optimization.
\newblock In \emph{International Conference on Machine Learning}, pages 49--58,
  2016{\natexlab{a}}.

\bibitem[Pomerleau(1989{\natexlab{a}})]{Pomerleau88}
Dean~A Pomerleau.
\newblock {ALVINN}: An autonomous land vehicle in a neural network.
\newblock In D~S Touretzky, editor, \emph{Advances in Neural Information
  Processing Systems 1}, pages 305--313. Morgan-Kaufmann, 1989{\natexlab{a}}.

\bibitem[Li et~al.(2017)Li, Song, and Ermon]{li2017infogail}
Yunzhu Li, Jiaming Song, and Stefano Ermon.
\newblock Infogail: Interpretable imitation learning from visual
  demonstrations.
\newblock In \emph{Advances in Neural Information Processing Systems}, pages
  3812--3822, 2017.

\bibitem[Ho and Ermon(2016)]{ho2016generative}
Jonathan Ho and Stefano Ermon.
\newblock Generative adversarial imitation learning.
\newblock In \emph{Advances in Neural Information Processing Systems}, pages
  4565--4573, 2016.

\bibitem[Nowozin et~al.(2016)Nowozin, Cseke, and Tomioka]{nowozin2016f}
Sebastian Nowozin, Botond Cseke, and Ryota Tomioka.
\newblock f-gan: Training generative neural samplers using variational
  divergence minimization.
\newblock In \emph{Advances in Neural Information Processing Systems}, pages
  271--279, 2016.

\bibitem[Osa et~al.(2018)Osa, Pajarinen, Neumann, Bagnell, Abbeel, and
  Peters]{ROB-053}
Takayuki Osa, Joni Pajarinen, Gerhard Neumann, J~Andrew Bagnell, Pieter Abbeel,
  and Jan Peters.
\newblock An algorithmic perspective on imitation learning.
\newblock \emph{arXiv preprint arXiv:1811.06711}, 2018.

\bibitem[Argall et~al.(2009)Argall, Chernova, Veloso, and
  Browning]{argall2009survey}
Brenna~D Argall, Sonia Chernova, Manuela Veloso, and Brett Browning.
\newblock A survey of robot learning from demonstration.
\newblock \emph{Robotics and autonomous systems}, 57\penalty0 (5):\penalty0
  469--483, 2009.

\bibitem[Billard et~al.(2016)Billard, Calinon, and
  Dillmann]{billard2016learning}
Aude~G Billard, Sylvain Calinon, and R{\"u}diger Dillmann.
\newblock Learning from humans.
\newblock In \emph{Springer handbook of robotics}, pages 1995--2014. Springer,
  2016.

\bibitem[Bagnell(2015)]{Bagnell-2015-5921}
J.~Andrew~(Drew) Bagnell.
\newblock An invitation to imitation.
\newblock Technical Report CMU-RI-TR-15-08, Carnegie Mellon University,
  Pittsburgh, PA, March 2015.

\bibitem[Ross and Bagnell(2014)]{ross2014reinforcement}
Stephane Ross and J~Andrew Bagnell.
\newblock Reinforcement and imitation learning via interactive no-regret
  learning.
\newblock \emph{arXiv preprint arXiv:1406.5979}, 2014.

\bibitem[Sun et~al.(2017)Sun, Venkatraman, Gordon, Boots, and
  Bagnell]{sun2017deeply}
Wen Sun, Arun Venkatraman, Geoffrey~J Gordon, Byron Boots, and J~Andrew
  Bagnell.
\newblock Deeply aggrevated: Differentiable imitation learning for sequential
  prediction.
\newblock In \emph{Proceedings of the 34th International Conference on Machine
  Learning-Volume 70}, pages 3309--3318. JMLR. org, 2017.

\bibitem[Sun et~al.(2018)Sun, Bagnell, and Boots]{sun2018truncated}
Wen Sun, J~Andrew Bagnell, and Byron Boots.
\newblock Truncated horizon policy search: Combining reinforcement learning \&
  imitation learning.
\newblock \emph{arXiv:1805.11240}, 2018.

\bibitem[Cheng et~al.(2018)Cheng, Yan, Wagener, and Boots]{cheng2018fast}
Ching-An Cheng, Xinyan Yan, Nolan Wagener, and Byron Boots.
\newblock Fast policy learning through imitation and reinforcement.
\newblock \emph{arXiv:1805.10413}, 2018.

\bibitem[Rajeswaran et~al.(2017)Rajeswaran, Kumar, Gupta, Vezzani, Schulman,
  Todorov, and Levine]{rajeswaran2017learning}
Aravind Rajeswaran, Vikash Kumar, Abhishek Gupta, Giulia Vezzani, John
  Schulman, Emanuel Todorov, and Sergey Levine.
\newblock Learning complex dexterous manipulation with deep reinforcement
  learning and demonstrations.
\newblock \emph{arXiv preprint arXiv:1709.10087}, 2017.

\bibitem[Pomerleau(1989{\natexlab{b}})]{pomerleau1989alvinn}
Dean~A Pomerleau.
\newblock Alvinn: An autonomous land vehicle in a neural network.
\newblock In \emph{Advances in neural information processing systems}, pages
  305--313, 1989{\natexlab{b}}.

\bibitem[Ross et~al.(2011)Ross, Gordon, and Bagnell]{ross2011reduction}
St{\'e}phane Ross, Geoffrey~J Gordon, and Drew Bagnell.
\newblock A reduction of imitation learning and structured prediction to
  no-regret online learning.
\newblock In \emph{AISTATS}, 2011.

\bibitem[Kim et~al.(2013)Kim, Farahmand, Pineau, and Precup]{kim2013learning}
Beomjoon Kim, Amir-massoud Farahmand, Joelle Pineau, and Doina Precup.
\newblock Learning from limited demonstrations.
\newblock In \emph{Advances in Neural Information Processing Systems}, pages
  2859--2867, 2013.

\bibitem[Gupta et~al.(2017)Gupta, Davidson, Levine, Sukthankar, and
  Malik]{gupta2017cognitive}
Saurabh Gupta, James Davidson, Sergey Levine, Rahul Sukthankar, and Jitendra
  Malik.
\newblock Cognitive mapping and planning for visual navigation.
\newblock In \emph{Proceedings of the IEEE Conference on Computer Vision and
  Pattern Recognition}, 2017.

\bibitem[Laskey et~al.(2017{\natexlab{a}})Laskey, Lee, Hsieh, Liaw, Mahler,
  Fox, and Goldberg]{laskey2017iterative}
Michael Laskey, Jonathan Lee, Wesley Hsieh, Richard Liaw, Jeffrey Mahler, Roy
  Fox, and Ken Goldberg.
\newblock Iterative noise injection for scalable imitation learning.
\newblock \emph{arXiv preprint arXiv:1703.09327}, 2017{\natexlab{a}}.

\bibitem[Laskey et~al.(2016)Laskey, Staszak, Hsieh, Mahler, Pokorny, Dragan,
  and Goldberg]{laskey2016shiv}
Michael Laskey, Sam Staszak, Wesley Yu-Shu Hsieh, Jeffrey Mahler, Florian~T
  Pokorny, Anca~D Dragan, and Ken Goldberg.
\newblock Shiv: Reducing supervisor burden in dagger using support vectors for
  efficient learning from demonstrations in high dimensional state spaces.
\newblock In \emph{Robotics and Automation (ICRA), 2016 IEEE International
  Conference on}, pages 462--469. IEEE, 2016.

\bibitem[Laskey et~al.(2017{\natexlab{b}})Laskey, Chuck, Lee, Mahler, Krishnan,
  Jamieson, Dragan, and Goldberg]{laskey2017comparing}
Michael Laskey, Caleb Chuck, Jonathan Lee, Jeffrey Mahler, Sanjay Krishnan,
  Kevin Jamieson, Anca Dragan, and Ken Goldberg.
\newblock Comparing human-centric and robot-centric sampling for robot deep
  learning from demonstrations.
\newblock In \emph{2017 IEEE International Conference on Robotics and
  Automation (ICRA)}. IEEE, 2017{\natexlab{b}}.

\bibitem[Ratliff et~al.(2009)Ratliff, Silver, and Bagnell]{ratliff2009learning}
Nathan~D Ratliff, David Silver, and J~Andrew Bagnell.
\newblock Learning to search: Functional gradient techniques for imitation
  learning.
\newblock \emph{Autonomous Robots}, 2009.

\bibitem[Ratliff et~al.(2006)Ratliff, Bagnell, and
  Zinkevich]{ratliff2006maximum}
Nathan~D Ratliff, J~Andrew Bagnell, and Martin~A Zinkevich.
\newblock Maximum margin planning.
\newblock In \emph{Proceedings of the 23rd international conference on Machine
  learning}, pages 729--736. ACM, 2006.

\bibitem[Piot et~al.(2017)Piot, Geist, and Pietquin]{piot2017bridging}
Bilal Piot, Matthieu Geist, and Olivier Pietquin.
\newblock Bridging the gap between imitation learning and inverse reinforcement
  learning.
\newblock \emph{IEEE transactions on neural networks and learning systems},
  28\penalty0 (8):\penalty0 1814--1826, 2017.

\bibitem[Abbeel and Ng(2004)]{abbeel2004apprenticeship}
Pieter Abbeel and Andrew~Y Ng.
\newblock Apprenticeship learning via inverse reinforcement learning.
\newblock In \emph{Proceedings of the twenty-first international conference on
  Machine learning}, page~1. ACM, 2004.

\bibitem[Ziebart et~al.(2008)Ziebart, Maas, Bagnell, and
  Dey]{ziebart2008maximum}
Brian~D Ziebart, Andrew~L Maas, J~Andrew Bagnell, and Anind~K Dey.
\newblock Maximum entropy inverse reinforcement learning.
\newblock In \emph{AAAI}, 2008.

\bibitem[Wulfmeier et~al.(2015)Wulfmeier, Ondruska, and
  Posner]{wulfmeier2015maximum}
Markus Wulfmeier, Peter Ondruska, and Ingmar Posner.
\newblock Maximum entropy deep inverse reinforcement learning.
\newblock \emph{arXiv preprint arXiv:1507.04888}, 2015.

\bibitem[Syed and Schapire(2008)]{syed2008game}
Umar Syed and Robert~E Schapire.
\newblock A game-theoretic approach to apprenticeship learning.
\newblock In \emph{Advances in neural information processing systems}, 2008.

\bibitem[Ho et~al.(2016)Ho, Gupta, and Ermon]{ho2016model}
Jonathan Ho, Jayesh Gupta, and Stefano Ermon.
\newblock Model-free imitation learning with policy optimization.
\newblock In \emph{International Conference on Machine Learning}, pages
  2760--2769, 2016.

\bibitem[Finn et~al.(2016{\natexlab{b}})Finn, Christiano, Abbeel, and
  Levine]{finn2016connection}
Chelsea Finn, Paul Christiano, Pieter Abbeel, and Sergey Levine.
\newblock A connection between generative adversarial networks, inverse
  reinforcement learning, and energy-based models.
\newblock \emph{arXiv preprint arXiv:1611.03852}, 2016{\natexlab{b}}.

\bibitem[Goodfellow et~al.(2014)Goodfellow, Pouget-Abadie, Mirza, Xu,
  Warde-Farley, Ozair, Courville, and Bengio]{goodfellow2014generative}
Ian Goodfellow, Jean Pouget-Abadie, Mehdi Mirza, Bing Xu, David Warde-Farley,
  Sherjil Ozair, Aaron Courville, and Yoshua Bengio.
\newblock Generative adversarial nets.
\newblock In \emph{Advances in neural information processing systems}, pages
  2672--2680, 2014.

\bibitem[Blond{\'e} and Kalousis(2018)]{blonde2018sample}
Lionel Blond{\'e} and Alexandros Kalousis.
\newblock Sample-efficient imitation learning via generative adversarial nets.
\newblock \emph{arXiv preprint arXiv:1809.02064}, 2018.

\bibitem[Fu et~al.(2017)Fu, Luo, and Levine]{fu2017learning}
Justin Fu, Katie Luo, and Sergey Levine.
\newblock Learning robust rewards with adversarial inverse reinforcement
  learning.
\newblock \emph{arXiv preprint arXiv:1710.11248}, 2017.

\bibitem[Qureshi and Yip(2018)]{qureshi2018adversarial}
Ahmed~H Qureshi and Michael~C Yip.
\newblock Adversarial imitation via variational inverse reinforcement learning.
\newblock \emph{arXiv preprint arXiv:1809.06404}, 2018.

\bibitem[Peng et~al.(2018{\natexlab{a}})Peng, Kanazawa, Toyer, Abbeel, and
  Levine]{peng2018variational}
Xue~Bin Peng, Angjoo Kanazawa, Sam Toyer, Pieter Abbeel, and Sergey Levine.
\newblock Variational discriminator bottleneck: Improving imitation learning,
  inverse rl, and gans by constraining information flow.
\newblock \emph{arXiv preprint arXiv:1810.00821}, 2018{\natexlab{a}}.

\bibitem[Torabi et~al.(2018{\natexlab{a}})Torabi, Warnell, and
  Stone]{torabi2018generative}
Faraz Torabi, Garrett Warnell, and Peter Stone.
\newblock Generative adversarial imitation from observation.
\newblock \emph{arXiv preprint arXiv:1807.06158}, 2018{\natexlab{a}}.

\bibitem[Torabi et~al.(2018{\natexlab{b}})Torabi, Warnell, and
  Stone]{torabi2018behavioral}
Faraz Torabi, Garrett Warnell, and Peter Stone.
\newblock Behavioral cloning from observation.
\newblock \emph{arXiv preprint arXiv:1805.01954}, 2018{\natexlab{b}}.

\bibitem[Peng et~al.(2018{\natexlab{b}})Peng, Kanazawa, Malik, Abbeel, and
  Levine]{peng2018sfv}
Xue~Bin Peng, Angjoo Kanazawa, Jitendra Malik, Pieter Abbeel, and Sergey
  Levine.
\newblock Sfv: Reinforcement learning of physical skills from videos.
\newblock In \emph{SIGGRAPH Asia 2018 Technical Papers}, page 178. ACM,
  2018{\natexlab{b}}.

\bibitem[Nguyen et~al.(2010)Nguyen, Wainwright, and
  Jordan]{nguyen2010estimating}
XuanLong Nguyen, Martin~J Wainwright, and Michael~I Jordan.
\newblock Estimating divergence functionals and the likelihood ratio by convex
  risk minimization.
\newblock \emph{IEEE Transactions on Information Theory}, 56\penalty0
  (11):\penalty0 5847--5861, 2010.

\bibitem[Boularias et~al.(2011)Boularias, Kober, and
  Peters]{boularias2011relative}
Abdeslam Boularias, Jens Kober, and Jan Peters.
\newblock Relative entropy inverse reinforcement learning.
\newblock In \emph{Proceedings of the Fourteenth International Conference on
  Artificial Intelligence and Statistics}, pages 182--189, 2011.

\bibitem[Rhinehart et~al.(2018)Rhinehart, Kitani, and
  Vernaza]{Rhinehart_2018_ECCV}
Nicholas Rhinehart, Kris~M. Kitani, and Paul Vernaza.
\newblock R2p2: A reparameterized pushforward policy for diverse, precise
  generative path forecasting.
\newblock In \emph{The European Conference on Computer Vision (ECCV)},
  September 2018.

\bibitem[Ghasemipour et~al.(2018)Ghasemipour, Gu, and Zemel]{Ghasemipour2019}
Seyed Kamyar~Seyed Ghasemipour, Shixiang Gu, and Richard Zemel.
\newblock Understanding the relation between maximum-entropy inverse
  reinforcement learning and behaviour cloning.
\newblock \emph{Workshop ICLR}, 2018.

\bibitem[Babes et~al.(2011)Babes, Marivate, Subramanian, and
  Littman]{babes2011apprenticeship}
Monica Babes, Vukosi Marivate, Kaushik Subramanian, and Michael~L Littman.
\newblock Apprenticeship learning about multiple intentions.
\newblock In \emph{Proceedings of the 28th International Conference on Machine
  Learning (ICML-11)}, pages 897--904, 2011.

\bibitem[Dimitrakakis and Rothkopf(2011)]{dimitrakakis2011bayesian}
Christos Dimitrakakis and Constantin~A Rothkopf.
\newblock Bayesian multitask inverse reinforcement learning.
\newblock In \emph{European Workshop on Reinforcement Learning}, pages
  273--284. Springer, 2011.

\bibitem[Chen et~al.(2016)Chen, Duan, Houthooft, Schulman, Sutskever, and
  Abbeel]{chen2016infogan}
Xi~Chen, Yan Duan, Rein Houthooft, John Schulman, Ilya Sutskever, and Pieter
  Abbeel.
\newblock Infogan: Interpretable representation learning by information
  maximizing generative adversarial nets.
\newblock In \emph{Advances in neural information processing systems}, pages
  2172--2180, 2016.

\bibitem[Hausman et~al.(2017)Hausman, Chebotar, Schaal, Sukhatme, and
  Lim]{hausman2017multi}
Karol Hausman, Yevgen Chebotar, Stefan Schaal, Gaurav Sukhatme, and Joseph~J
  Lim.
\newblock Multi-modal imitation learning from unstructured demonstrations using
  generative adversarial nets.
\newblock In \emph{Advances in Neural Information Processing Systems}, pages
  1235--1245, 2017.

\bibitem[Lee et~al.(2018{\natexlab{a}})Lee, Choi, and Oh]{lee2018maximum}
Kyungjae Lee, Sungjoon Choi, and Songhwai Oh.
\newblock Maximum causal tsallis entropy imitation learning.
\newblock In \emph{Advances in Neural Information Processing Systems},
  2018{\natexlab{a}}.

\bibitem[Lee et~al.(2018{\natexlab{b}})Lee, Choi, and Oh]{lee2018sparse}
Kyungjae Lee, Sungjoon Choi, and Songhwai Oh.
\newblock Sparse markov decision processes with causal sparse tsallis entropy
  regularization for reinforcement learning.
\newblock \emph{IEEE Robotics and Automation Letters}, 2018{\natexlab{b}}.

\bibitem[Belousov and Peters(2017)]{belousov2017f}
Boris Belousov and Jan Peters.
\newblock f-divergence constrained policy improvement.
\newblock \emph{arXiv preprint arXiv:1801.00056}, 2017.

\bibitem[Csisz{\'a}r and Shields(2004)]{csiszar2004information}
Imre Csisz{\'a}r and Paul~C Shields.
\newblock \emph{Information theory and statistics: A tutorial}.
\newblock Now Publishers Inc, 2004.

\bibitem[Liese and Vajda(2006)]{liese2006divergences}
Friedrich Liese and Igor Vajda.
\newblock On divergences and informations in statistics and information theory.
\newblock \emph{IEEE Transactions on Information Theory}, 2006.

\bibitem[Kanamori et~al.(2012)Kanamori, Suzuki, and
  Sugiyama]{kanamori2012statistical}
Takafumi Kanamori, Taiji Suzuki, and Masashi Sugiyama.
\newblock Statistical analysis of kernel-based least-squares density-ratio
  estimation.
\newblock \emph{Machine Learning}, 86\penalty0 (3):\penalty0 335--367, 2012.

\bibitem[Zhang et~al.(2019)Zhang, Bird, Habib, Xu, and
  Barber]{zhang2019variational}
Mingtian Zhang, Thomas Bird, Raza Habib, Tianlin Xu, and David Barber.
\newblock Variational f-divergence minimization.
\newblock \emph{arXiv preprint arXiv:1907.11891}, 2019.

\bibitem[Sun et~al.(2019)Sun, Vemula, Boots, and Bagnell]{sun2019provably}
Wen Sun, Anirudh Vemula, Byron Boots, and J~Andrew Bagnell.
\newblock Provably efficient imitation learning from observation alone.
\newblock \emph{arXiv preprint arXiv:1905.10948}, 2019.

\end{thebibliography}
